\documentclass{article}
 \def\<{\langle}
 \def\>{\rangle}
 \usepackage{subfig}
 \usepackage{listings}

 \usepackage{kbordermatrix}
 \usepackage[utf8x]{inputenc} 
 \usepackage{subfig}

\usepackage{hyperref}
 \usepackage{xcolor, colortbl}
 \usepackage{float}
 \usepackage{epstopdf}
 \usepackage{color}
 \usepackage{dsfont}
 \usepackage{eqnarray}
 \usepackage{multirow}
 \usepackage{array}
 \usepackage{natbib}
 \usepackage{enumerate}
 \usepackage{amsmath,amssymb}
 \usepackage{graphicx,indentfirst,amsmath}
 \usepackage{booktabs}
 \usepackage{enumitem}
 \usepackage{lipsum}
 \usepackage{dsfont}
 \usepackage{algorithm}
 \usepackage{algorithmic}
 \usepackage{caption}

 \usepackage{graphics}
 \usepackage{latexsym}
 \usepackage{geometry}
 \usepackage{placeins}
 \usepackage{epsf,verbatim}
 \usepackage{epstopdf}
 \usepackage{multirow}
 \usepackage{pgf,pgfarrows,pgfnodes}
 \usepackage{tikz}
 \usetikzlibrary{matrix,arrows,decorations.pathmorphing}
 \usetikzlibrary{shapes,arrows}
 \usepackage{color}
 \usepackage{epsfig}
 \usetikzlibrary{patterns}
 \usetikzlibrary{shapes}
 \usetikzlibrary{arrows}
 \usepackage{tikz-3dplot}

\def\M{\mathcal{M}}
\def\N{\mathcal{N}}

\def\<{\langle}
\def\>{\rangle}

\def\beq{\begin{eqnarray*}}
\def\eeq{\end{eqnarray*}}

\definecolor{LRed}{rgb}{1,.8,.8}
\newtheorem{remark}{Remark}[section]
\newtheorem{theorem}{Theorem}[section]
\newtheorem{definition}[theorem]{Definition}
\newtheorem{lemma}[theorem]{Lemma}
\newtheorem{proof}[theorem]{Proof}

\newcommand\ci{\perp\!\!\!\perp}

\DeclareMathOperator*{\argmax}{arg\,max}

\title{The scalable Birth-Death MCMC Algorithm for Mixed Graphical Model Learning with Application to Genomic Data Integration}
\author{ Nanwei Wang \\
  Lunenfeld-Tanenbaum Research Institute\\
\and
 Laurent Briollais \\
  Lunenfeld-Tanenbaum Research Institute\\
  \and
 Helene Massam \\
  York University\\
}
\date{}
\begin{document}
\maketitle


  

\begin{abstract}
Recent advances in biological research have seen the emergence of high-throughput technologies with numerous applications that allow the study of biological mechanisms at an unprecedented depth and scale. A large amount of genomic data  is now distributed through consortia like The Cancer Genome Atlas (TCGA), where specific types of biological information on specific type of tissue or cell are available. In cancer research, the challenge is now to perform integrative analyses of high-dimensional multi-omic data with the goal to better understand genomic processes that correlate with cancer outcomes, e.g. elucidate gene networks that discriminate a specific cancer subgroups (cancer sub-typing) or discovering gene networks that overlap across different cancer types (pan-cancer studies). In this paper, we propose  a novel mixed graphical model approach to analyze multi-omic data of different types (continuous, discrete and count) and perform model selection by extending the Birth-Death MCMC (BDMCMC) algorithm initially proposed by \citet{stephens2000bayesian} and later developed by \cite{mohammadi2015bayesian}.  We compare the performance of our method to the LASSO method and the standard BDMCMC method using simulations and find that our method is superior in terms of both computational efficiency and the accuracy of the model selection results. Finally, an application to the TCGA breast cancer data shows that integrating genomic information at different levels (mutation and expression data) leads to better subtyping of breast cancers.  
\end{abstract}

\section{Introduction}
Recent advances in biological research have seen the emergence of high-throughput technologies with numerous applications in genetics, genomics and other ``--omic" disciplines  (i.e. transcriptomic, proteomic, metabolomic, etc.). The goal of these studies is to provide a richer and unbiased view of the cell at different biological levels. The cell is often represented as a complex system with various levels of biological information communicating with each other. The DNA found within each cell contains the genetic blueprint for the entire organism. Each gene contains the information necessary to instruct the cellular machinery how to make mRNA, and in turn the protein encoded by the order of DNA bases constituting the gene. Each one of these proteins is responsible for carrying out one or more specified molecular functions within the cell. Differing patterns of gene expression (i.e., different mRNA and protein levels) in different tissues can explain differences in both cellular function and appearance. The new high-throughput technologies allow the study of biological mechanisms at an unprecedented depth and scale. 

These large amount of  genomic and other -omic data have been distributed through freely available public international consortia like The Cancer Genome Atlas (TCGA) \citet{TCGA2013}, The Encyclopedia of DNA Elements
\citet{ENCODE2011}, and The NIH Roadmap Epigenomics Mapping Consortium (Roadmap) \citet{Roadmap2010}. The Cancer Genome Atlas (TCGA), for example, is a National Institute of Health (NIH) initiative, it makes publicly available molecular and clinical information for more than 30 types of human cancers including exome (variant analysis), single nucleotide polymorphism (SNP), DNA methylation, transcriptome (mRNA), microRNA (miRNA) and proteome.

Each consortium encompasses specific types of biological information on specific type of tissue or cell, but the challenge now, is to analyze them together. In cancer research, such analysis provides an invaluable opportunity for research laboratories to better understand the developmental progression of normal cells to cancer state at the molecular level and importantly, correlate these phenotypes with tissue of origins. Thus, increasing research attention is being paid to the integrative analysis and modeling of various types of biomedical data.

The problem of integrating different types of omics data raised important statistical challenges. First, each single type of data is typically high-dimensional, meaning that the number of variables is often larger than
the number of observations. Analyzing high-dimensional data of one type is already challenging, yet analyzing multiple sources of data raises additional problems that often cannot be addressed by standard statistical methods. 
Second,the variables are often of ‘mixed types’ (e.g., continuous, count-valued, discrete, skewed continuous, bounded, among others). For example in integrative genomics, genotype data is typically
discrete, gene expression as measured via RNA-sequencing is count-valued or non-negative skewed continuous, DNA methylation data is bounded on the interval zero to one, and so forth.  Finally, biologists are often interested in some types of structure of this data, e.g. in the form or pathways and networks, and making inference about these structures in a high-dimensional space is also quite challenging.

In the context of multi-omics data integration, various statistical methods have been proposed, the two main approaches are penalized regression variable selection and Bayesian variable selection methods. Denote the disease outcome $Y$, which can be continuous,  discrete,  or time-to-event and $X$, the design matrix of $p$-dimensional genomics features such as SNP genotypes, DNA methylation and gene expressions,
among other omics measurements for $n$ individuals. In this “large p, small n” problem, the compelling challenge is to identify important features that are associated with disease, e.g. cancer outcome. 

{\it Mixed graphical models.}
In this paper, we will use graphical models to identify dependencies between given variables. Let $G=(V,E)$ denote an undirected graph, where $V=\{1,2,\cdots,p\}$ denotes the vertex set and $E\subseteq 2^V$ denotes the set of edges. Let $X=(X_v, v\in V)$ denote the p-dimensional random vector where each variable $X_v$ is indexed by $v\in V$. The distribution of $X$ is said to be Markov with respect to graph $G$ when
\[
(i,j)\not \in E \Longrightarrow X_i \ci X_j |X_{V \setminus \{i,j\}}.
\] 
In the literature, two types of undirected graphical models are well studied: Gaussian graphical model for multivariate normal data and discrete log-linear model for discrete data. Recently graphical models for mixed data, which we refer to as mixed graphical models here, become quite relevant, especially for genomic studies. The study of mixed graphical models can be traced back to the conditional Gaussian density in \citet{lauritzen1996graphical}. Indeed, in such studies, variables of various types, such as Gaussian, binary or counting need to be considered simultaneously and therefore dependencies determined Since then, there have been several papers about mixed graphical models, including \citet{fellinghauer2013stable}, \citet{lee2015learning} and \citet{cheng2017high}. In 2014, \citet{yang2014mixed} proposed mixed graphical models via an exponential family distribution and applied this model to invasive breast carcinoma data from the TCGA database. Mixed graphical models are an extension of graphical models to problems involving mixed types of variables, a situation which is quite common in real-world problems.

{\it Penalization or regularization} plays an important role in statistical inference of high-dimensional  omics data, especially in the context of multi-dimensional integration studies. As the number of variables is often larger than the number of observations, it is very likely that an over fitting problem can occur, i.e. select models with many false discovery variables and make incorrect inference. A penalized model can be formulated as

$$\hat\beta = \textrm{argmin}_\beta \{ L(\beta; Y, X) + \textrm{pen}(\beta; \lambda) \}, $$
where $L(.)$ is the loss function measuring the fit of the model and pen() is the penalty function controlling the sparsity of the model through the data-dependent tuning parameter $\lambda$. The most classical high-dimensional model selection method is the Lasso proposed by \citet{tibshirani1996regression}. The $l_1$ penalty on parameters will shrink the small MLE values of some parameters to zeros and result in sparse models to avoid over fitting. Since then, a lot of researchers have been working on various penalty variable selection methods, such as smoothly clipped absolute deviation (SCAD) penalty \citet {fan2001variable}, Adaptive Lasso \citet{zou2006adaptive} and minimax concave penalty (MCP) \citet{zhang2010nearly}. All theses methods can perform well under sparsity assumption, but model uncertainty remains a big challenge for these various penalized regression method, especially in today's big genomic data era. 

 In the {\it Bayesian} approach to model selection, we write $\M=\{M_i,i=1,\cdots,k\}$ for the model space and put a prior distribution $\pi(M)$ on $\M$. Then if $\theta \in R^d$ is the parameter of $M$, we put a prior distribution $p(\theta)$ on the d-dimensional parameter space. The data is generated from the distribution $p(x|\theta)$. Given the data $D$, we can compute the posterior distribution of models as follows:
 \begin{equation}
 \label{eq:post}
 p(M|D)=\frac{p(D|M)\pi(M)}{p(D)},
\end{equation}  
where 
\begin{equation}
p(D|M)=\int p(D|\theta)p(\theta|M)d\theta
\end{equation} 
is the marginal likelihood of the data $D$, sometimes called evidence. Ignoring the denominator in \eqref{eq:post},  we use the following formula:
\begin{equation}
\label{eq:post2}
 p(M|D)\propto p(D|M)\pi(M).
\end{equation} 
Knowing the posterior distribution $p(M|D)$, we have several ways to conduct the model selection process. One can try to maximize $p(M|D)$ to find the mode of the distribution. However, most of the time, we do not have the exact formula of $P(M|D)$, or if we do, it is too difficult to maximize. The Bayes factor is another method, which is widely used. It compares models pairwise. The Bayes factor between two models $M_1,\ M_2$ is the ratio of the evidence of two models:
\[
B_{12}=\frac{p(D|M_1)}{P(D|M_2)}=\frac{p(M_1|D)}{p(M_2|D)} \times \frac{\pi(M_2)}{\pi(M_1)}.
\] 
A large value of the Bayes factor $B_{12}$ means that model $M_1$ is preferred over $M_2$. However, pairwise comparison is impossible when the model space is large. Therefore we need to use an MCMC method to draw samples from the posterior distribution $p(M|D)$. Two MCMC methods are often used to explore statistical models with different dimensions: the Reversible Jump MCMC and the Birth-Death MCMC (\citet{stephens2000bayesian}). The review paper from \citet{cappe2002reversible} compared these two methods and studied their statistical properties. 

Recently, \citet{mohammadi2015bayesian} applied the Birth-Death MCMC method (henceforth abbreviated BDMCMC) to learn sparse Gaussian graphical models. The same method was applied to study discrete log-linear models in \citet{dobra2018loglinear}. The BDMCMC methodology that we present here differs from the BDMCMC methodology in these two papers in two major ways: first, we apply the BIC value (see \citet{wasserman2000bayesian}) and extended BIC value (see \citet{chen2008extended}) to approximate the marginal likelihood $P(D|M)$; second, a neighourhood structure learning is proposed to estimate the graphical model structure, rather than directly estimating its global structure. Both of these two modifications make the computation for the BDMCMC algorithm easier and more flexible in a high-dimensional setting. We name our method Scalable Birth-Death MCMC (henceforth abbreviated SBDMCMC) algorithm.

The remainder of this paper is organized as follows: in Section 2, we provide a review of  model selection in regressions and graphical models and define the new mixed graphical model; The SBDMCMC algorithm is described in details in Section 3; in Section 4, we  discuss the computation of SBDMCMC. Numerical simulations are presented in Section 5. The paper concludes with a real data analysis in Section 6.

\section{Preliminaries and notation}
In this section, we define our notation and recall basic concepts of graphical models.

We use  capital letters (e.g. X,Y,Z) to denote random variables, their corresponding small letters (e.g. x,y,z) to denote their values. The l-norm of the vector $\theta$ is denoted $||\theta||_l$. For example $||\theta||_1=\sum_i|\theta_i|$, $||\theta||_2=\sqrt{\sum_i\theta_i^2}$. The indicator function of a set A is $\mathbf{1}_{A}(a)=1$ for $a\in A$ and $\mathbf{1}_{A}(a)=0$ for $a\not \in A$. $|A|$ denotes the cardinality of the set A.

The undirected graph $G=(V,E)$ is  defined by its vertex set $V$ and its set of undirected edges $E\subseteq V\times V$. A subset $C\in V$ is said to be a clique if for any $i,j\in C$, $(i,j)\in E$. The neighbourhood of a vertex r is $N_r=\{t\in V|(t,r)\in E\}$. Consider a random vector $X=(X_v, v\in V)$ with components indexed by $V$, we say that the distribution of $X$ follows the pairwise Markov property with respect to $G$ if
\begin{equation}
(i,j)\not \in E \Rightarrow X_i \ci X_j|X_{V\setminus \{i,j\}}.
\label{eq:markov}
\end{equation}

If the density of $X$ is strictly positive everywhere on its domain, then all Markov properties local, global, pairwise and factorization are equivalent. Moreover by the Hammersley-Clifford theorem, we can factorize the distribution of $X$ as follows:
\[
P(x)=\frac{1}{z(\phi)} \prod_{C\in \mathcal{C}} \Phi_{C}(x_C),
\] 
where $\mathcal{C}$ is the set of cliques, $\Phi_{C}(X_C),C\in \mathcal{C}$ are potential functions that only depend on the variables in clique $C$, $z(\phi)=\int \prod_{C\in \mathcal{C}} \Phi_{C}(x_C)dx$ is called the partition function or normalization constant. We can also write the density function $P(X)$ in exponential family form:
\[
P(x)=\exp(\sum_{C\in \mathcal{C}}\theta_CB_C(x_C)-k(\theta)),
\]
where $\Phi_{C}(x_C)=\exp(\theta_CB_C(x_C))$ and $k(\theta)$ is the log of the partition function.

There are two main classes of graphical models: Gaussian  and  log-linear discrete models.
If $X$ follows the Multivariate normal distribution $\N(0, \Sigma)$, we can write the probability density function as follows:
\[
f(x|K)=\exp\{\<-\frac{1}{2}xx^t,K\> -(\frac{p}{2}\log (2\pi)-\frac{1}{2}det(K)) \},
\] 
where $K=\Sigma^{-1}$ is the canonical parameter, $\<\cdot ,\cdot \>$ denotes the inner product between the canonical parameter, $-\frac{1}{2}xx^t$ is the sufficient statistic and $(\frac{p}{2}\log (2\pi)-\frac{1}{2}det(K))$ is the log-partition function. For Gaussian models, \eqref{eq:markov} is equivalent to $(\Sigma^{-1})_{ij}=0 \text{ for all } (i,j)\not \in E.$ The distribution of $X$ belongs to the Gaussian graphical model Markov with respect to $G$
\[
\N_G=\{\N_p(0,\Sigma)|(\Sigma^{-1})_{ij}=0, \forall (i,j)\in E\}.
\]

If the data is discrete, i.e. if $X_v,v\in V$ takes its values in a finite set and $X$ follows a so-called log-linear model, then the distribution of $X$ belongs to the graphical model Markov with respect to $G$ if its density can be written as
\[
p(x)=\frac{1}{z(\theta)}\exp(\sum_{C\in \mathcal{C}} \theta_Cx_C).
\]
In particular, if the edges are the only cliques in $G$ and $X_v,v\in V$ are binary variables, then the probability mass function of X is given as follows:
\begin{equation}
\label{eq:loglinear}
p(x)=\frac{1}{z(\theta)} \exp(\sum_{v=1}^p\theta_vx_v+\sum_{(i,j)\in E}\theta_{ij}x_ix_j),
\end{equation}
where $z(\theta)=\sum_{x\in \{0,1\}^p}\exp(\sum_{v=1}^p\theta_vx_v+\sum_{(i,j)\in E}\theta_{ij}x_ix_j)$ is the partition function. 

\section{Mixed graphical models}
Recently, due to the variety of variable types considered in big data problems, mixed graphical models have received a lot of attention. The earliest study related to binary and Gaussian variables is the $Conditional$ $Gaussian$ $density$ proposed by \citet{lauritzen1996graphical}. In 2014, \citet{yang2014mixed} studied the the mixed graphical models via exponential family distributions. \citet{cheng2017high} modified the conditional Gaussian density models and gave a simplified density formula for Gaussian and binary mixed graphical models. In this section, we will first introduce the definitions of mixed graphical models given in \citet{cheng2017high} and \citet{yang2014mixed}, then we will talk about the type of mixed graphical models we use in this paper.  

\subsection{Conditional Gaussian distribution}

In the $Conditional$ $Gaussian$ $density$ models, the continuous variables follow a Gaussian distribution conditioned on discrete variables. Following this idea, other papers such as \citet{fellinghauer2013stable}, \citet{lee2015learning} and \citet{cheng2017high} focus on mixed graphical models for Gaussian and discrete random variables. As pointed out in \citet{cheng2017high},  the two former papers can be seen as special cases of the model in \citet{cheng2017high}, that we describe now.

Let $\{Z_1,Z_2,\cdots,Z_q\}$, $\{Y_1,Y_2,\cdots,Y_q\}$ denote binary variables and Gaussian variables, respectively, \citet{cheng2017high} proposed the following $Conditional$ $Gaussian$ $density$:
\begin{equation}
\label{eq:cond}
\log f(z,y) = (\lambda_0+\sum_{i}\lambda_iz_i+\sum_{i>j}\lambda_{ij}z_iz_j)+ y^t(\eta_0+\sum_{i}\eta_iz_i)-\frac{1}{2}y^t(\Phi_0+\sum_{i}\Phi_iz_i)y,
\end{equation}
where the diagonal values of the matrix $\Phi_i$ are zeros and $\lambda_0$ is the normalizing constant.
There are two simplifications in the density \eqref{eq:cond}: 
\begin{enumerate}
\item The term $\sum_{i>j}\lambda_{ij}z_iz_j$ shows that only up to two-way interactions between binary variables are considered;
\item The conditional mean and covariance matrix of the continuous variables $y$ are linear functions of binary variables $z$.
\end{enumerate}

Given the density function in \eqref{eq:cond}, the conditional distribution of any binary random variable $Z_i$ given the other binary variables $Z_{-i}$ and continuous variables $y$ can be expressed through the following logistic regression:
\[
\log \frac{p(z_i=1|z_{-i},y)}{p(z_i=0|z_{-i},y)}=\lambda_i+\sum_{j\not=i}\lambda_{ij}z_j+\sum_{\gamma=1}^p \eta_{i\gamma}y_{\gamma}-\sum_{r=1}^p\sum_{\mu=1}^p\frac{1}{2}\Phi_i^{\gamma \mu}y_{\gamma}y_{\mu},
\] 
and the conditional distribution of $Y$ given $Z$ is a multivariate Gaussian distribution with conditional mean and covariance matrix:
\[
\begin{array}{lcl}
E(Y|Z) &=& (\Phi_0+\sum_{i}\Phi_iz_i)^{-1}(\eta_0+\sum_{i}\eta_iz_i), \\
Cov(Y|Z) &=& (\Phi_0+\sum_{i}\Phi_iz_i)^{-1}.
\end{array}
\] 

In order to perform model selection in the class of conditional Gaussian distributions given in \eqref{eq:cond}, \citet{cheng2017high} proposed to use a regularized regression method for every variable, i.e. logistic regression for a binary variable and linear regression for a continuous variable.  

\subsection{Mixed graphical models via exponential family distribution}
\citet{yang2014mixed} proposed a general class of mixed graphical models: mixed graphical models via exponential family distributions.
Let $X=(X_v,v\in V)$ denote the random vector of all types of variables. \citet{yang2014mixed} assume that, conditionally on  the other variables $X_{-r}$, any variable $X_r$  follows a univariate exponential family distribution with density:
\begin{equation}
P(X_r|X_{-r})=\exp\{E_r(X_{-r})B_r(X_r)+C_r(X_r)-D_r(X_{-r})\},
\end{equation}
where the functions $E_r,B_r,C_r,D_r$ are determined by the choice of exponential family, such as Gaussian, Bernoulli or Poisson. By Theorem 1 in \citet{yang2014mixed} it holds that, if and only if this product has the form
\begin{equation}
E_r(X_{-r})B_r(X_r)=\theta_r+\sum_{t\in V\setminus r}\theta_{rt}B_t(X_t),
\label{eq:heat}
\end{equation}
then these conditional distributions are consistent with the following joint distribution
\begin{equation}
P(X;\theta)=\exp\{ \sum_{r\in V}\theta_rB_r(X_r)+ \sum_{r\not=t}\theta_{rt}B_r(X_r)B_t(X_t)+\sum_{r\in V}C_r(X_r)-k(\theta)   \}.
\label{eq:viaexp}
\end{equation}
Equations \eqref{eq:heat} and \eqref{eq:viaexp} are slightly different from those in \citet{yang2014mixed}. In order to simplify the notations, we only consider up to two-way interactions and those can easily be extended to higher order interactions. We can easily derive  formula \eqref{eq:cond} from \eqref{eq:viaexp} by an appropriate choice of the function $B_r,r\in V$.

For model selection purposes, \citet{yang2014mixed} also proposed to use regularized regression of each variable on other variables independently and  then to combine all the non-zero regression coefficients to get the graph structure of the mixed graphical model.

\subsection{Local mixed graphical models}
Even though \citet{yang2014mixed} gave the general form of the joint distribution for mixed graphical models, there are some limitations to their parametrization: the conditional Gaussian mixed graphical model can only deal with Gaussian and discrete variables while the mixed graphical model via exponential family has constraints with respect to parameters as shown in Table 1 of \citet{chen2014selection}. In this section, we will define a new type of mixed graphical model, which is applicable to high-dimensional mixed data and also possesses good statistical properties. We call this type of mixed graphical models local mixed graphical models. \\

Let $W=(W_v, v\in V)$ denote the random vector of mixed types of variables  which follows the Markov property with respect to an undirected graph $G=(V,E)$. In the following, we define the mixed graphical model locally based on the conditional distribution of each variable given the other mixed variables without specifying their joint distribution.
\begin{definition}
The local mixed graphical model is a series of conditional distributions for each variable $W_v$ given the other variables $W_{-v}$, and we assume the conditional distributions follow an exponential family distribution with density of the form
\begin{equation}
p(W_v|W_{-v})=\exp(\theta_vW_v+\sum_{j\in V\setminus v} \theta_{vj} W_vW_j+A(W_v)-k(\theta)).
\label{eq:mixed}
\end{equation}

\end{definition} 
\begin{remark}
The local mixed graphical model is an extension of the local Poisson graphical model of \citet{allen2013local}. In a high-dimensional setting, the fact that the global model follows a given probability distribution is difficult to satisfy and often requires additional constraints on the parameter space.  The other issue in specifying a global model in a high-dimensional setting is related to model learning and parameter estimation. These two problems are intractable through global MLE when the dimension of the model is very high. Therefore, even though the joint distribution of mixed graphical models is given for a global model in both papers \citet{cheng2017high} and \citet{yang2014mixed}, the model learning strategy is based on a penalized conditional likelihood for each variable, i.e, the strategy is based on local model estimates to approximate the global model parameters. For the performance of various distributed composite likelihood estimates, see \citet{massam2018local} who studied their convergence properties and showed how numerically close they are to the global MLE.
\end{remark}
In general, $W_v$ can be any type of random variable that follows an exponential family distribution. In this paper, we only consider three types of common variables: Gaussian, binary, and count. Assuming that there are $p_1$ Gaussian variables denoted by $X_i$, $p_2$ binary variables denoted by $Y_j$, $p_3$ counting variables denoted by $Z_k$. The total number of variables is $p=p_1+p_2+p_3$. We use the notation $\theta^X_i,\theta^Y_j,\theta^Z_k$ for the marginal effect parameters of the three types of variables, and the notation $\theta^{XY}_{ij},\theta^{XZ}_{ik},\theta^{YY}_{j_1j_2},\theta^{YZ}_{jk},\theta^{ZZ}_{k_1k_2}$ for the interaction effect parameters.
\begin{itemize}
\item For Gaussian variables $X$, the conditional distribution of $X|Y,Z$ follows a multivariate normal distribution $MVN(u(Y,Z),\Sigma)$. Here we assume that the covariance matrix of the Gaussian variables $X$ is fixed, and the mean is a linear function of $Y,Z$:
\[
u(Y,Z)=E(X|Y,Z)=\theta^X+\sum_{j}\theta^{XY}_{\bullet j} y_j+\sum_{k} \theta^{XZ}_{\bullet k}z_k,
\]
where $\theta^X=[\theta^X_i, \ i=1,\cdots,p_1],\theta^{XY}_{\bullet j}=[\theta^{XY}_{i j}, i=1,\cdots,p_1]$ and $\theta^{XZ}_{\bullet k}=[\theta^{XZ}_{i k}, i=1,\cdots,p_1]$ 
\item For each binary variable $Y_j$, the conditional distribution is a Bernoulli distribution with mean:
\[
u(X,Y_{-j},Z)=E(Y_j|X,Y_{-j},Z)=\frac{\exp(\theta^{Y}_j+\sum_{j_2}\theta^{YY}_{jj_2}y_{j_2}+\sum_{i}\theta^{XY}_{ij}x_i+\sum_{k}\theta^{YZ}_{jk}z_k)}{1+\exp(\theta^{Y}_j+\sum_{j_2}\theta^{YY}_{jj_2}y_{j_2}+\sum_{i}\theta^{XY}_{ij}x_i+\sum_{k}\theta^{YZ}_{jk}z_k)}
\]
or in generalized linear regression models, we write the following equation
\[
\log \frac{p(y_j=1|x,y_{-j},z)}{p(y_j=0|x,y_{-j},z)}=\theta^{Y}_j+\sum_{j_2\not =j}\theta^{YY}_{jj_2}y_{j_2}+\sum_{i}\theta^{XY}_{ij}x_i+\sum_{k}\theta^{YZ}_{jk}z_k
\]
\item For a counting variable $Z_k$, we assume that the conditional distribution follows a Poisson distribution with mean: 
\[
u(X,Y,Z_{-k})=E(Z_k|X,Y,Z_{-k})=\exp(\theta^{Z}_k+\sum_{k_2\not =k}\theta^{ZZ}_{kk_2}z_{k_2}+\sum_{i}\theta^{XZ}_{ik}x_i+\sum_{j}\theta^{YZ}_{jk}y_j).
\]
\end{itemize}
We note that the interactions between variables in the neighbourhood of each conditional distribution are not included. It is equivalent to say that we do not consider the three-way or higher interactions in the corresponding global mixed graphical models, i.e. we are working with the class of Ising models.

\section{Graphical model selection methods}
Graphical model selection is a classical problem in statistics. For different types of graphical models, there are different ways to perform model selection, we briefly describe some popular methods below.

\subsection{Model selection based on likelihood}
\citet{friedman2008sparse} first proposed the Graphical Lasso to estimate the inverse covariance matrix $K$ in Gaussian graphical models. The problem is to maximize the penalized log-likelihood in  Gaussian graphical models:
\[
K= \argmax_{K } \{\log det(K)-\<S,K\>-\lambda ||K||_1\},
\]
where S is the sample covariance matrix, $\lambda$ is the penalty term to determines the sparsity of the estimation. There are different types of penalty functions and different approaches for estimating the inverse covariance matrix. \\

For discrete graphical models, the Graphical Lasso is not suitable, as the evaluation of the normalizing constant  $k(\theta)=\sum_{x\in \mathcal{X}} \exp(\<\theta,t(x)\>)$ is an NP-hard problem. In this regard, pseudo-likelihood is traditionally used as a substitute to the global likelihood function. \citet{ravikumar2010high} used the $l_1$-regularized logistic regression to learn the neighbourhood structure for every binary variable and then combined all these neighbourhoods to obtain the global graph structure. 
This idea can be summarized as follows:
\begin{enumerate}
\item For every node $v\in V$, first build the conditional log-likelihood function
\[
l_v(\theta^v)=\sum_{i=1}^n \log(p(x_v^{(i)}|x_{V\setminus v}^{(i)})).
\]
\item Get the parameter estimate $\hat{\theta^v}$ by maximizing the following regularized log-likelihood function:
\[
\theta^v= \argmax \ l_v(\theta^v) -\lambda \lVert\theta^v\rVert_1.
\]  
\item The non-zero elements in $\hat{\theta^v}$ can give us the neighbourhood of $v$: $N_v$.
\item Get the global graph structure based on AND or OR rule from all the pseudolikelihood estimates $\hat{\theta}^v, v\in V$. For each pairwise parameter $\theta_{uv}$, the non-zero estimates are given by one of the following two rules: 
\begin{itemize}
\item AND rule:
\[
\hat{\theta}_{uv}\begin{cases} \not =0 \quad \text{if }\hat{\theta}^v_u \not=0 \text{ and } \hat{\theta}^u_v \not=0 , \\ =0, \quad \text{otherwise} \end{cases} 
\]

\item OR rule:
\[
\hat{\theta}_{uv}\begin{cases} \not=0 \quad \text{if }\hat{\theta}^v_u \not=0 \text{ or } \hat{\theta}^u_v \not=0 , \\ =0, \quad \text{otherwise} \end{cases} 
\]

\end{itemize}

\end{enumerate}

For mixed graphical models, the only neighbourhood selection proposed thus far is based on  $l_1$-regularized regression, see \citet{yang2014mixed} and \citet{cheng2017high}. Below, we propose a Bayesian variable selection approach.

\subsection{Bayesian graphical model selection} 
The Bayesian framework can be described as follows: given a set of variables $X$, we first define a prior distribution $p(G_i),i=1,2\cdots,k$ over the finite space of graphs. ${\cal{G}}=\{G_1,G_2,\cdots,G_k\}$. Second, given a graph $G$, denote $\pi(\theta|G)$ the parameter prior distribution. Lastly, the data $D=[x^{(1)},x^{(2)},\cdots,x^{(n)}]$ is generated from the distribution $p(X|\theta)$. Based on Bayes' theorem, the posterior distribution of a graph $G$ given data $D$ is
\begin{equation}
\label{posterior}
p(G|D)=\frac{p(D|G)p(G)}{\sum_{j=1}^k p(D|G_j)p(G_j)},
\end{equation}
where $$p(D|G_j)=\int \prod_{i=1}^n p(x^{(i)}|\theta)\pi(\theta|G_j) d\theta,$$ is the marginal likelihood of the data given the model $G_j$. It is also called the evidence of model $G_j$. Based on the posterior distribution on ${{\cal G}}$, one can perform pairwise comparison between  any two graphs $G_1$, $G_2$:
\[
\frac{p(G_1|D)}{p(G_2|D)}=\frac{p(D|G_1)}{p(D|G_2)}\times \frac{p(G_1)}{p(G_2)}.
\]

When the graph structure space is very large, pairwise model comparison is not efficient to find the best model. Markov Chain Monte Carlo (MCMC) can be used to simulate samples from the posterior distribution. These samples can then be used to find the model with largest posterior probability or to average the models with high posterior probabilities. \\

The BDMCMC algorithm was first proposed   for  Gaussian graphical model selection in \citet{mohammadi2015bayesian}. They studied the Birth-Death MCMC method for Gaussian graphical models: they first sampled graph structures using a Poisson process, then sampled the  inverse covariance matrix $K$ from the G-wishart distribution. A similar BDMCMC approach was applied to discrete log-linear model selection in \citet{dobra2018loglinear}. For mixed graphical models, there is no posterior distribution we can sample the parameters from. In fact, Bayesian model selection methods have not yet been proposed for mixed graphical models. In the next section, we present our modified BDMCMC method to learn mixed graphical models.

\section{A Scalable Birth-Death MCMC algorithm for mixed graphical models}

In this section we introduce a new local Bayesian model selection method which we call the Scalable Birth-Death MCMC (henceforth abbreviated SBDMCMC) in details. 
While the BDMCMC method samples the global graph structure as in \citet{mohammadi2015bayesian} and \citet{dobra2018loglinear}, we will use our method to select the neighbourhood of each variable.\\

Consider the mixed variables indexed by the vertex set V, our goal is to fit a mixed graphical model $G=(V,E)$ given a dataset $D$ from the distribution of $X$. For high dimensional problems, a neighbourhood selection approach has a large computational advantage over a global graph structure learning method. The neighbourhood learning process is as follows:
\begin{enumerate}
\item For every variable $X_v$, consider all the variables in $X_{V\setminus v}$ as the candidate covariates to build the neighbourhood of $X_v$, then the neighbourhood structure space is the power set of $X_{V\setminus v}$, which we denote as $\mathcal{P}(X_{V\setminus v})$. The neighbourhood learning method seeks to find the neighbourhood structure $N_v$ and to build the conditional distribution $p(X_v|X_{N_v})$ given the data $D$;
\item Given a neighbourhood $N_v\in \mathcal{P}(X_{V\setminus v})$, we derive the posterior distribution:
\[
p(N_v|D)\propto p(D|N_v)p(N_v),
\]
where $p(D|N_v)=\int \prod_{i=1}^n p(x_v^{(i)}|x_{N_v}^{(i)};\theta)\pi(\theta|N_v) d\theta $;

\item  Design a MCMC algorithm with stationary distribution $p(N_v|D)$ and get a sample of $N_v$: $\{N_v^{(1)},N_v^{(2)},\cdots,N_v^{(n)}\}$. Use Bayesian model averaging to decide which variables are in the neighbourhood of $X_v$:

\[
p(u\in N_v)=\sum_{i}^{n} \mathbf{1}_{N_v^{(i)}}(u)p(N_v^{(i)}|D), \quad u\in V\setminus v.
\]
In this paper, we set the probability value threshold as $0.5$, i.e.

\[
\begin{cases}u\in N_v \quad \mbox{if}\quad p(u\in N_v)\geq 0.5 \\ u\not \in N_v \quad \mbox{if}\quad p(u\in N_v)< 0.5. \end{cases}
\]
\item  Get the global graph structure based on the AND or OR rule from all the neighbourhoods $N_v, v\in V$:
\begin{itemize}
\item AND rule:
\[
\{u,v\} \begin{cases} \in E \quad \text{if } u\in N_v \text{ and } v \in N_u, \\ \not \in E, \quad \text{otherwise} \end{cases} 
\]

\item OR rule:
\[
\{u,v\} \begin{cases} \in E \quad \text{if } u\in N_v \text{ or } v \in N_u, \\ \not \in E, \quad \text{otherwise}. \end{cases} 
\]
\end{itemize}

\end{enumerate}

Step 3 is where we need the Birth-Death MCMC algorithm. BDMCMC is a continuous time Markov process in the neighbourhood space. This process explores the space by adding and removing variables corresponding to birth and death jumps. Given the current neighbourhood $N_v$, the birth and death events are defined by the following independent Poisson process:
\begin{itemize}
	\item Birth event: each node $u \not \in N_v$ is born independently of the other variables as a Poisson process with rate $B_{u}(N_v)$. If this birth event of variable $X_{u}$ happens, the process jumps to the new state: $N_v\cup u$ 
	\item Death event: each node $u \in N_v$ dies independently of the other variables as a Poisson process with rate $D_{u}(N_v)$. If this death event of variable $X_{u}$ happens, the process jumps to the new state: $N_v\setminus u$.
\end{itemize}

The time to the next birth/death jump follows the exponential distribution with mean $$\lambda=\frac{1}{\sum_{u\not \in N_v}B_{u}(N_v)+\sum_{u \in N_v}D_{u}(N_v)},$$

and the probability of the birth and death events are respectively
\[
\begin{array}{lcl}
p_{N_v}(u) &=& \frac{B_{u}(N_v)}{\sum_{u\not \in N_v}B_{u}(N_v)+\sum_{u \in N_v}D_{u}(N_v)}, \quad u\not \in N_v \\
q_{N_v}(u) &=& \frac{D_{u}(N_v)}{\sum_{u\not \in N_v}B_{u}(N_v)+\sum_{u \in N_v}D_{u}(N_v)}, \quad u\in N_v 
\end{array}
\]

Here, we use the notation $p_{N_v}(u)$ to denote the probability of adding variable $u$ to neighbourhood $N_v$, i.e. $p_{N_v}(u)$ is the probability of this Markov process jumping from $N_v$ to $N_v\cup u$. Similarly, $q_{N_v}(u)$ is the probability of this Markov process jumping from $N_v$ to $N_v\setminus u$.   With the above probabilities, we can define the transition kernel probability matrix $K$. Notice that in the Birth-Death Markov chain, the jump can only happen between two neighbour states that differ by one variable. i.e. $K_{N_{v} N_{v} }=0$ and $K_{N_v N_{u}}=0$, where $| N_v\setminus N_{u}|>1$ or $| N_{u}\setminus N_v|>1$. In order to simplify the notation, we will write $K_{vv}\text{ for }K_{N_{v} N_{v} }$ and $K_{u v}\text{ for }K_{N_{u} N_{v} }$. Let us give a small example here. Suppose there are two variables $X_1,X_2$ in $X_{V\setminus v}$, the neighbourhood space (or states of the Markov chain) is $\{\emptyset,\{X_1\},\{X_2\},\{X_1,X_2\}\}$, the transition kernel probability matrix $K$ is then equals to
\[
\kbordermatrix{
    & \emptyset & \{X_1\} & \{X_2\} &  \{X_1,X_2\} \\
    \emptyset & 0 & p_{\emptyset}(X_1) & p_{\emptyset}(X_2) & 0  \\
    \{X_1\} & q_{\{X_1\}}(X_1) & 0 & 0 & p_{\{X_1\}}(X_2)  \\
    \{X_2\} & q_{\{X_2\}}(X_2) & 0 & 0 & p_{\{X_2\}}(X_1)  \\
    \{X_1,X_2\} & 0 & q_{\{X_1,X_2\}}(X_2)  & q_{\{X_1,X_2\}}(X_1) & 0 
  }
\]
Now we recall two important concepts in Markov chains.
\begin{definition}
\label{def:stationary}
Let $\pi$ denote the distribution of the states, we say $\pi$ is the stationary distribution of the Markov chain, if 
\begin{equation}
\pi=\pi K
\label{eq:stationary}
\end{equation}

\end{definition}

Next we give the definition of the detailed balanced equation for a Markov chain.
\begin{definition}
Let $\pi$ be the stationary distribution, then the Markov chain is said to be reversible or to satisfy the detailed balanced condition if 
\begin{equation}
\pi_x K_{xy}=\pi_yK_{yx},
\label{eq:detailed}
\end{equation}

where $x,y$ are any two states of the Markov chain.
\end{definition}

Note that the detailed balanced condition is stronger than the condition that $\pi$ has a stationary distribution. So we have the following lemma:
\begin{lemma}
Let $K$ be the transition kernel probability matrix, then \eqref{eq:detailed} implies \eqref{eq:stationary}
\end{lemma}
\begin{proof}
\[
\sum_{x}\pi_x K_{xy}=\sum_{x}\pi_yK_{yx}=\pi_y\;\;\forall y,
\]
and thus $\pi K=\pi$.
\end{proof}
To ensure that the BDMCMC converges to the posterior distribution of the neighbourhood structure given the data, we give the following theorem.
\begin{theorem}
Given the SBDMCMC process as described in step 4, let p be the number of variables and let $\mathcal{P}(X_{V\setminus v})$ be the state space.
If the birth/death rates are defined as follows: for any neighbourhood structure $N_v\in \mathcal{P}(X_V\setminus v)$,
\begin{equation}
\begin{array}{lcl}
B_{u}(N_v) &=&\cfrac{1}{p-1} \cfrac{p(N_v\cup u|D)}{p(N_v|D)}  \quad \forall u\not \in N_v\\[20pt]

D_{u}(N_v) &=& \cfrac{1}{p-1} \cfrac{p(N_v\setminus u|D)}{p(N_v|D)} \quad \forall u \in N_v ,
\end{array}
\label{eq:bdrate}
\end{equation}
 then the stationary distribution of the above BDMCMC is $\{p(N_v|D),N_v\in \mathcal{P}(X_V\setminus v)\}$.
\end{theorem}

\begin{proof}
We need to prove that $p(N_v|D)$ satisfies equation \eqref{eq:stationary}, i.e.,
\[
p(N_v|D)=\sum_{N^{'}_v \in \mathcal{P}(X_V\setminus v)}p(N^{'}_v |D)K_{N^{'}_v ,N_v }.
\]
As mentioned above, not all the entries of $K_{N^{'}_v ,N_v }$ are non-zeros. We have
\[
\begin{array}{lcl}
K_{N_v\cup u,N_v }&=& D_u(N_v\cup u), \quad \forall u\not \in N_v, \\[20pt]
K_{N_v\setminus u,N_v }&=& B_u(N_v\setminus u), \quad \forall u \in N_v,
\end{array}
\]
and all other entries are zero. Therefore, since the transition probabilities involving the addition or deletion of more than one vertex are equal to zero, we have
\[
\begin{array}{lcl}
\sum_{N^{'}_v \in \mathcal{P}(X_V\setminus v)}p(N^{'}_v |D)K_{N^{'}_v ,N_v }&=&
\sum_{u \not \in N_v}p(N_v\cup u|D)K_{N_v\cup u,N_v } +\sum_{u\in N_v}p(N_v\setminus u|D)K_{N_v\setminus u,N_v } \\[10pt]
&=& \sum_{u \not \in N_v} p(N_v\cup u|D)\frac{1}{|X_{V\setminus v}|} \frac{p(N_v|D)}{p(N_v\cup u|D)} \\
& &+\sum_{u\in N_v}p(N_v\setminus u|D) \frac{1}{|X_{V\setminus v|}}\frac{p(N_v|D)}{p(N_v\setminus u|D)}\\[10pt]
&=&\frac{1}{p-1}( \sum_{u\not \in N_v}p(N_v|D)+\sum_{u \in N_v}p(N_v|D))\\[10pt]
&=&  p(N_v|D).
\end{array}
\]

\end{proof}

\section{Computation of the Birth-Death rate}
To simplify the notation, let $r_u(N_v^{(0)},\theta_0)=\frac{p(N_v^{(1)}|D)}{p(N_v^{(0)}|D)}$ denote the Birth or Death rate from the current neighbourhood structure $N_v^{(0)}$ jump to the new neighbourhood structure $N_v^{(1)}$. We have 
\[r_u(N_v^{(0)},\theta_0)=
\begin{cases}B_u(N_v^{(0)},\theta_0) \quad & u\not\in N_v^{(0)}, \\
D_u(N_v^{(0)},\theta_0) \quad & u\in N_v^{(0)}.
\end{cases}
\]
In our setting, here the model denotes the neighbourhood structure of a given random variable $X_v$.

In this section, we will offer a fast and accurate estimation of $r_u(N_v^{(0)},\theta_0)$:
\[
r_u(N_v^{(0)},\theta_0)=\frac{p(N_v^{(1)}|D)}{p(N_v^{(0)}|D)}=\frac{p(D|N_v^{(1)})}{p(D|N_v^{(0)})}\times \frac{p(N_v^{(1)})}{p(N_v^{(0)})},
\]
so the change rate $r_u$ is the product of Bayes factor $BF(M_1,M_0)$ and the ratio of model priors.

\subsection{BIC and Extended-BIC}
In most cases, computing the exact Bayes factor value is difficult. In regression models, we can use the Bayesian information criterion (BIC) proposed in \citet{schwarz1978estimating} to approximate $p(D|M)$.
\[
\log p(D|M)= l(\hat{\beta})-\frac{d}{2}\log n +\mathcal{O}(1),
\]
where $l(\hat{\beta})$ is the log-likelihood function evaluated at the MLE of $\beta$, $d$ is the dimension of the regression, i.e. the length of $\beta$ and $n$ is the sample size. Therefore, we can use the BIC value $BIC(M)=-2l(\hat{\beta})+d\log(n)$ to approximate the marginal likelihood:
\[
p(D|M)\approx \exp (-BIC(M)/2).
\]
This approximation does not require any  integration nor does it depend on the prior of the parameters in the model. The error term $\mathcal{O}(1)$ does not converge to 0 as $n \to \infty$, but it has a relatively small value compared to $\log p(D|M)$ as $n\to \infty$. However, for small-n-large-p problems, which is the main focus of this paper, the BIC becomes overly liberal and fails to perform variable selection. \citet{chen2008extended} proposed the Extended Bayesian Information Criteria(E-BIC), which takes into account both the number of unknown parameters and the complexity of the model space, for small-$n$-big-$p$ problems. The E-BIC formula is as follows:
\begin{equation}
EBIC(M,\theta)=-2\log L(\theta))+d\log(n)+2\gamma\log \tau(M), \ 0\leq\gamma \leq 1,
\end{equation}
where $\tau(M)$ is the number of models with the same dimension $d$ as model $M$. In regression models, assume the number of variables in model $M$ is $k$, then $\tau(M)={p \choose k}$.

\subsection{Priors on model space $\M$}
In small-n-big-p problems, it seems more natural to choose priors on the model space that to put more weight on models with fewer variables, i.e. we would like to choose a prior that gives us a smaller neighbourhood of each variable $X_v$. We offer three options in this paper: 

\begin{itemize}
\item[1.] Let $k$ denote the number of variables in the neighbourhood structure model $M$, the prior is 
\[
p(M)\propto a^k, a\in (0,1].
\]
The prior is similar to the one given in \citet{dobra2018loglinear}. A smaller $a$ value will put less prior weight on a dense model. We will use the notation "DM prior" for this prior.\\

\item[2.] The second prior is a modified version of one of the prior presented by \citet{nan2014variable}:
\[
p(M) \propto \exp(-a C_M),
\]
where $0\leq a \leq 1$,  $C_M=\log{p \choose k}+2\log(k)$.
This prior is similar to the E-BIC idea. It takes the model complexity into consideration.  The term $C_M$ plays a similar role to the term $\tau(M)$ in E-BIC. This prior will be denoted the "NY prior".\\
\item[3.] The third prior is given by \citet{scott2010bayes}:
\[
p(M)=a^{k}(1-a)^{p-k}, \quad 0\leq a\leq 1.
\]
Its rationale is to treat the variable inclusion as exchangeable Bernoulli trials with common success probability $a$. This prior will be denoted the "SB prior." In the real data analysis Section \ref{sec:real}, we will only use the DM prior with $a=0.5$
\end{itemize}

All three priors are used and compared in the sensitive analysis in Section \ref{sec:prior}. In the real data analysis (Section \ref{sec:real}), we only use the DM prior with $a=0.5$

\section{Numerical Simulations}
We compare the model selection performance of our SBDMCMC method with the standard BDMCMC method given in the R package \textit{"BDgraph"} of \citet{mohammadi2019bdgraph} and with a neighbourhood selection approach based on $l_1$-penalized regression method (MGM) given in the R package \textit{"‘mgm"} of \citet{haslbeck2015mgm} under various scenarios. Here, an $F_1$-score \citet{baldi2000assessing} is used to evaluate the performance of the three methods:
\[
F_1=\frac{2TP}{2TP+FP+FN},
\]
where TP, FP, FN are the true positive, false positive, false negative rates, respectively.

\subsection{Simulation of Gaussian graphical models}
Scale free networks and random networks are the most common networks in real life. In this section, we simulate a Gaussian graphical model Markov with respect to a scale free network and a random network as shown in Figure \ref{fig:gn} and Figure \ref{fig:gn2}. 

\begin{figure}[H]
\subfloat[Scale free network with 100 variables]{\includegraphics[width=0.5\linewidth]{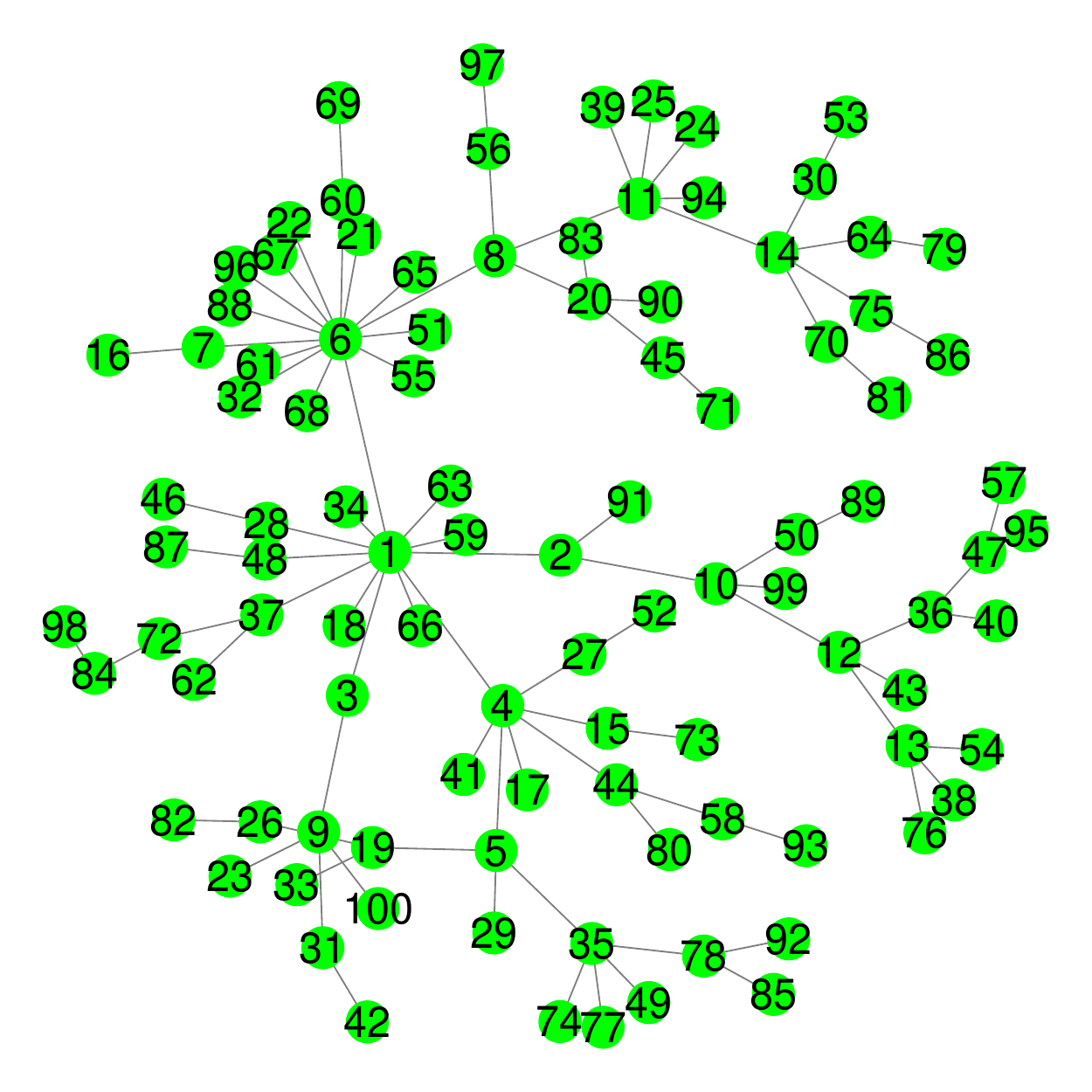}\label{fig:gn}}
\subfloat[Random network with 100 variables]{\includegraphics[width=0.5\linewidth]{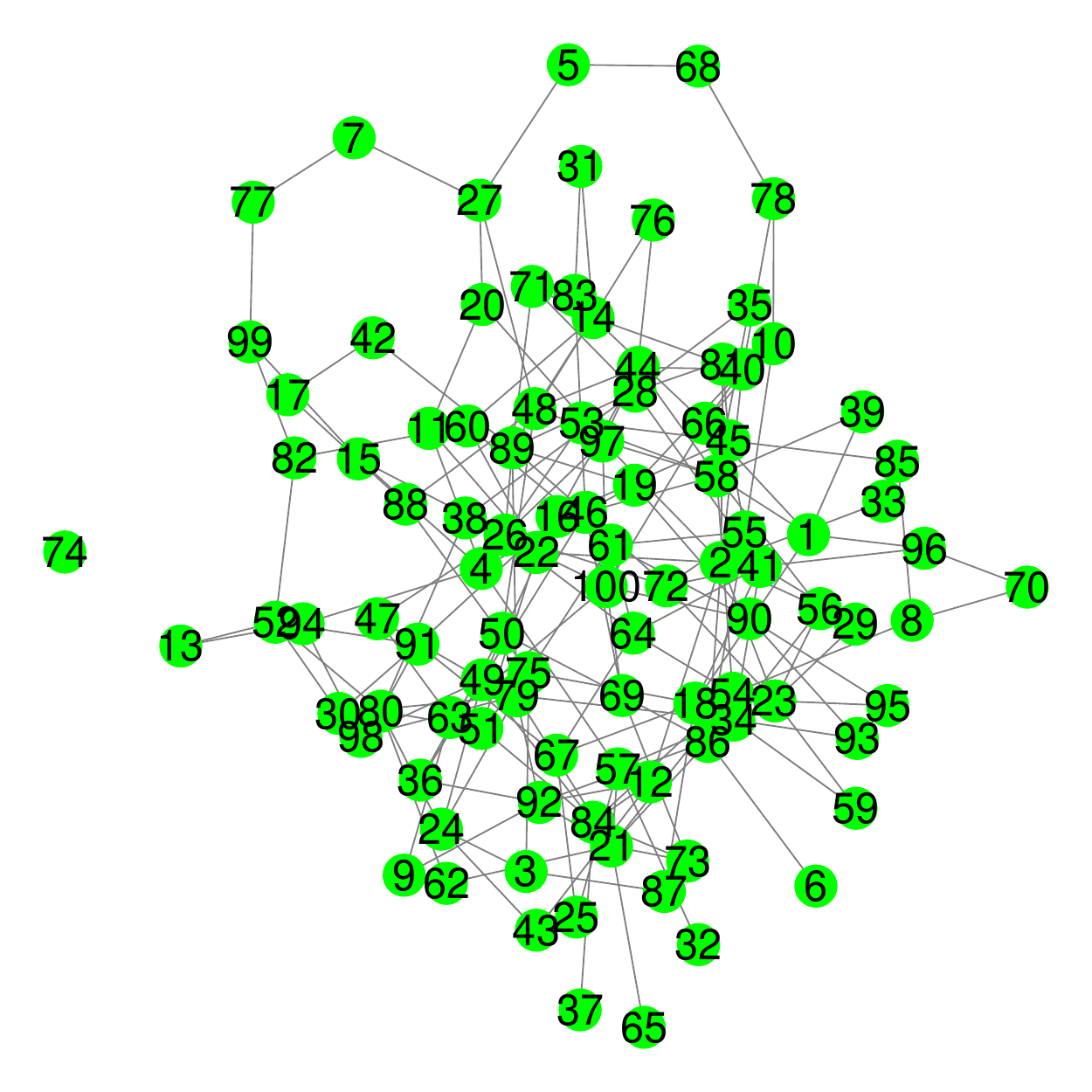}\label{fig:gn2}}
\caption{Graph structures used in the simulation of Gaussian graphical models}
\end{figure}

The setting of the simulations in Gaussian graphical model is as follows:
\begin{enumerate}
\item Generate the population inverse covariance matrix $K$: Given the graph structure $G$ as shown in Figure \ref{fig:gn} and Figure \ref{fig:gn2},  generate the matrix $K$ for the model Markov with respect to $G$ based on the conditions: $$ \begin{array}{lcl}
 K_{ij}=0 &,\ & (i,j)\not \in E,\\
  K_{ij}\not =0 &,\ & (i,j) \in E.
\end{array}$$
For the non-zero entries of $K$, we generate a random number from the standard normal distribution. 
\item Generate data $D$ from the multivariate normal distribution $\N(0, K^{-1})$ with various sample sizes equal to $200,500,1000,2000$ or $3000$.
\item Perform model selection with three methods: SBDMCMC, BDMCMC, MGM given the data $D$, and compute the $F_1$ score for each of the three methods.
\item Repeat Steps 3 and 4 fifty times to get and average value of the $F_1$ score.
\end{enumerate}

The results of these simulations are shown in Figure \ref{fig:gaussianf1} and Figure \ref{fig:gaussianf2}. From the plot, we can see that our SBDMCMC method is the best out of the three methods for the scale free network. For the random network, SBDMCMC is a little worse  than BDMCMC, but better than the MGM method. 
\begin{figure}[H]
\subfloat[scale-free network ]
{\includegraphics[width=0.5 \linewidth]{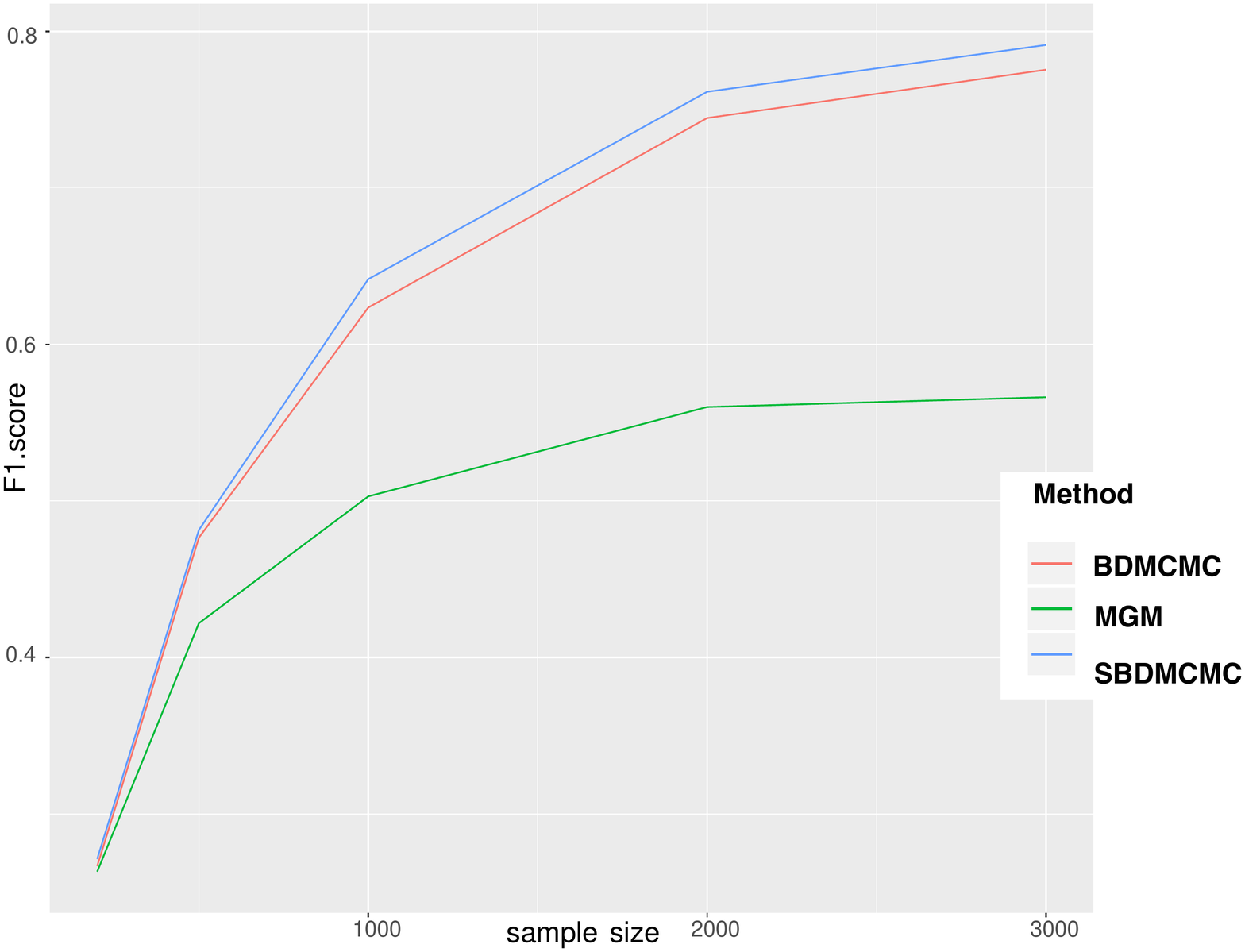} \label{fig:gaussianf1}}
\subfloat[Random network  ]
{\includegraphics[width=0.5 \linewidth]{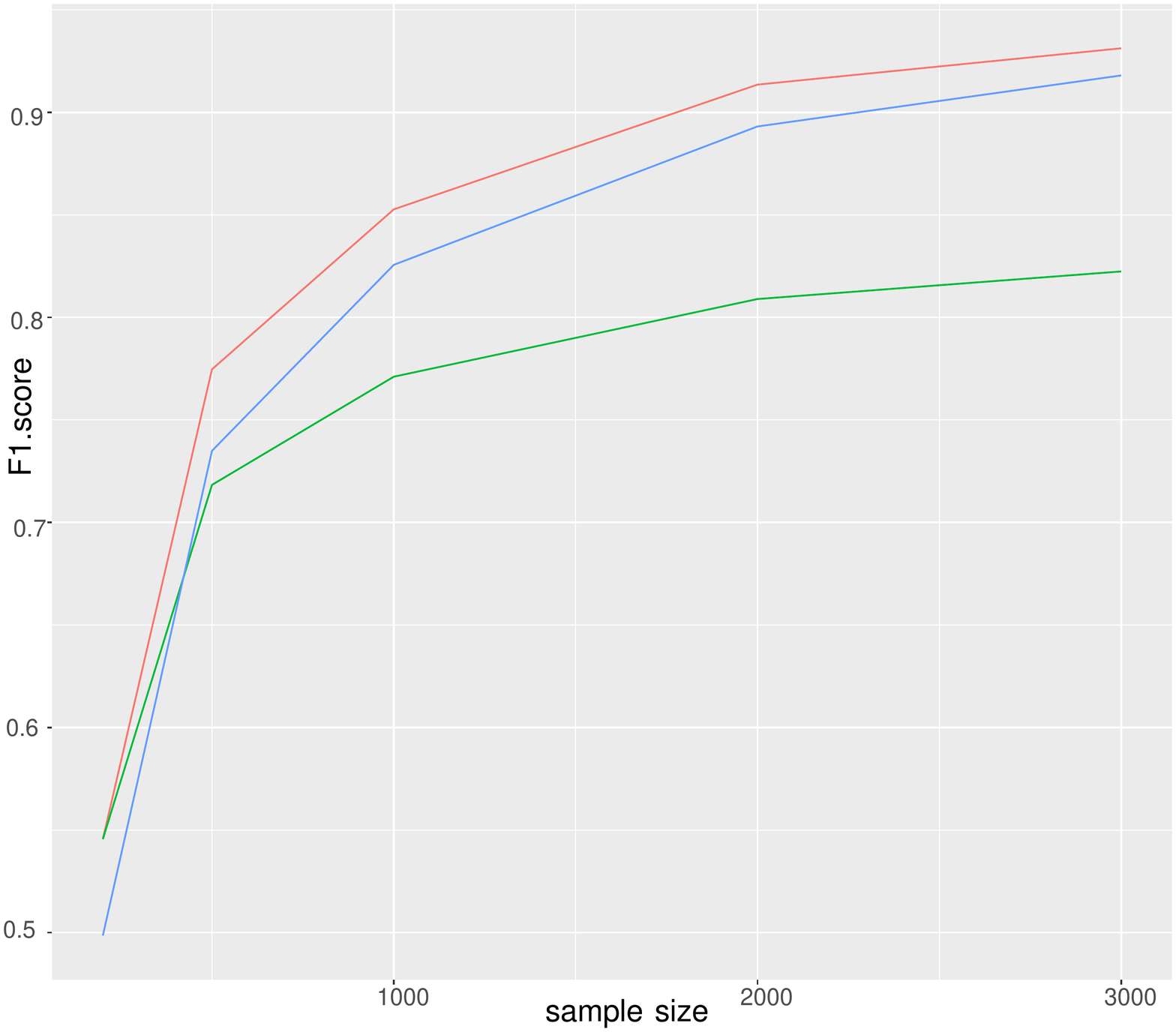} \label{fig:gaussianf2}}
\caption{$F_1$ score comparison in Gaussian graphical models}
\end{figure}

\subsection{Simulation of discrete graphical models}

The simulation of discrete graphical models  is very similar to the simulation of Gaussian graphical models. First, we generate a scale free network graph and a random network graph $G=(V,E)$ as show in Figure \ref{fig:gn} and \ref{fig:gn2}, second, we sample log-linear parameters $\theta=\{\theta_v,v\in V,\theta_{uv},(u,v)\in E\}$ from a standard normal distribution, third, we sample data with different sample sizes from the log-linear model with probability density function \eqref{eq:loglinear}, in the end, we use different methods to recover the generating models.

The results of this simulation is given in  Figure \ref{fig:binaryf1} and Figure \ref{fig:binaryf2}. Our SBDMCMC performs the best out of three methods for discrete graphical models. 
\begin{figure}[H]
\subfloat[Scale-free network  ]
{\includegraphics[width=0.5 \linewidth]{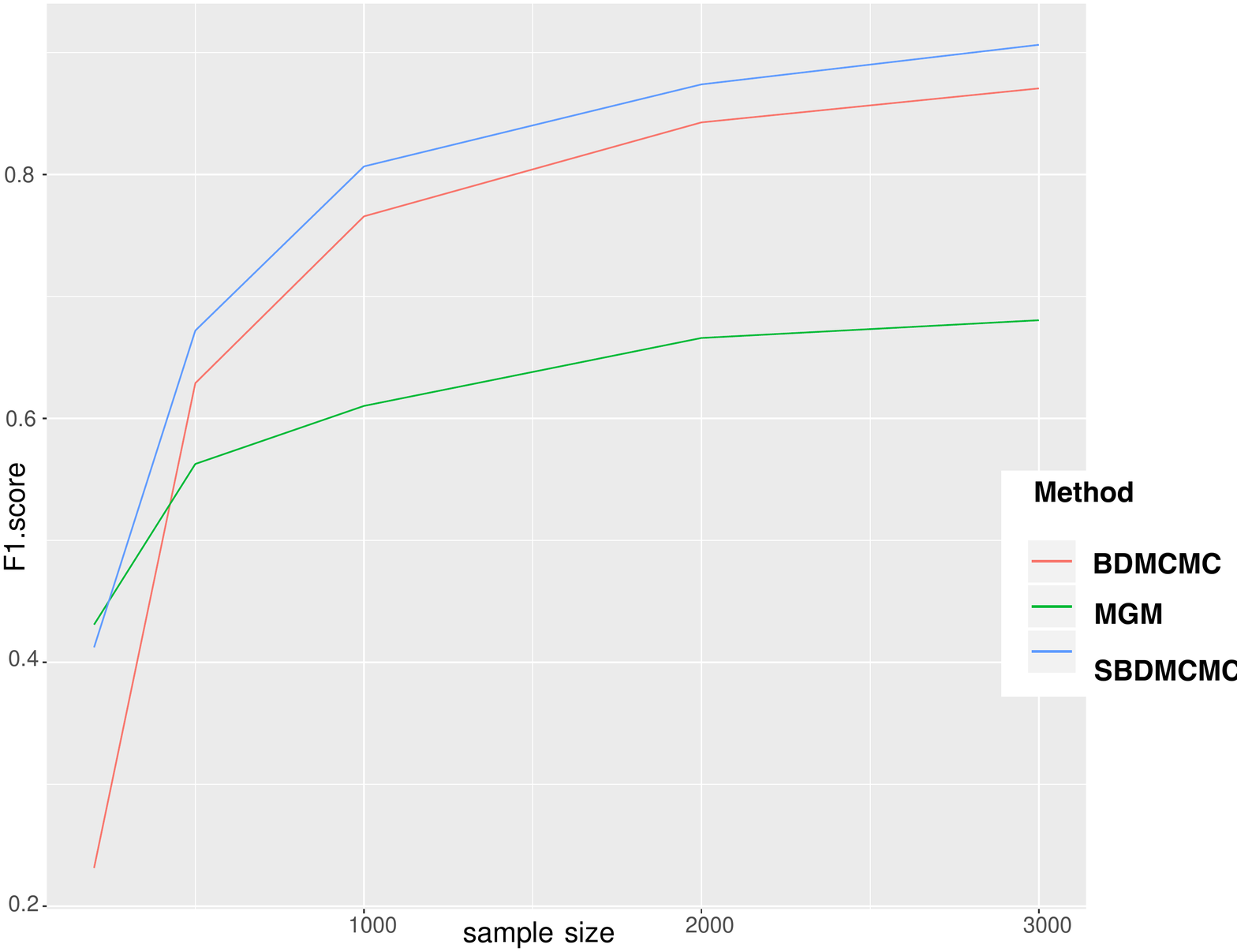} \label{fig:binaryf1}}
\subfloat[Random network ]
{\includegraphics[width=0.5 \linewidth]{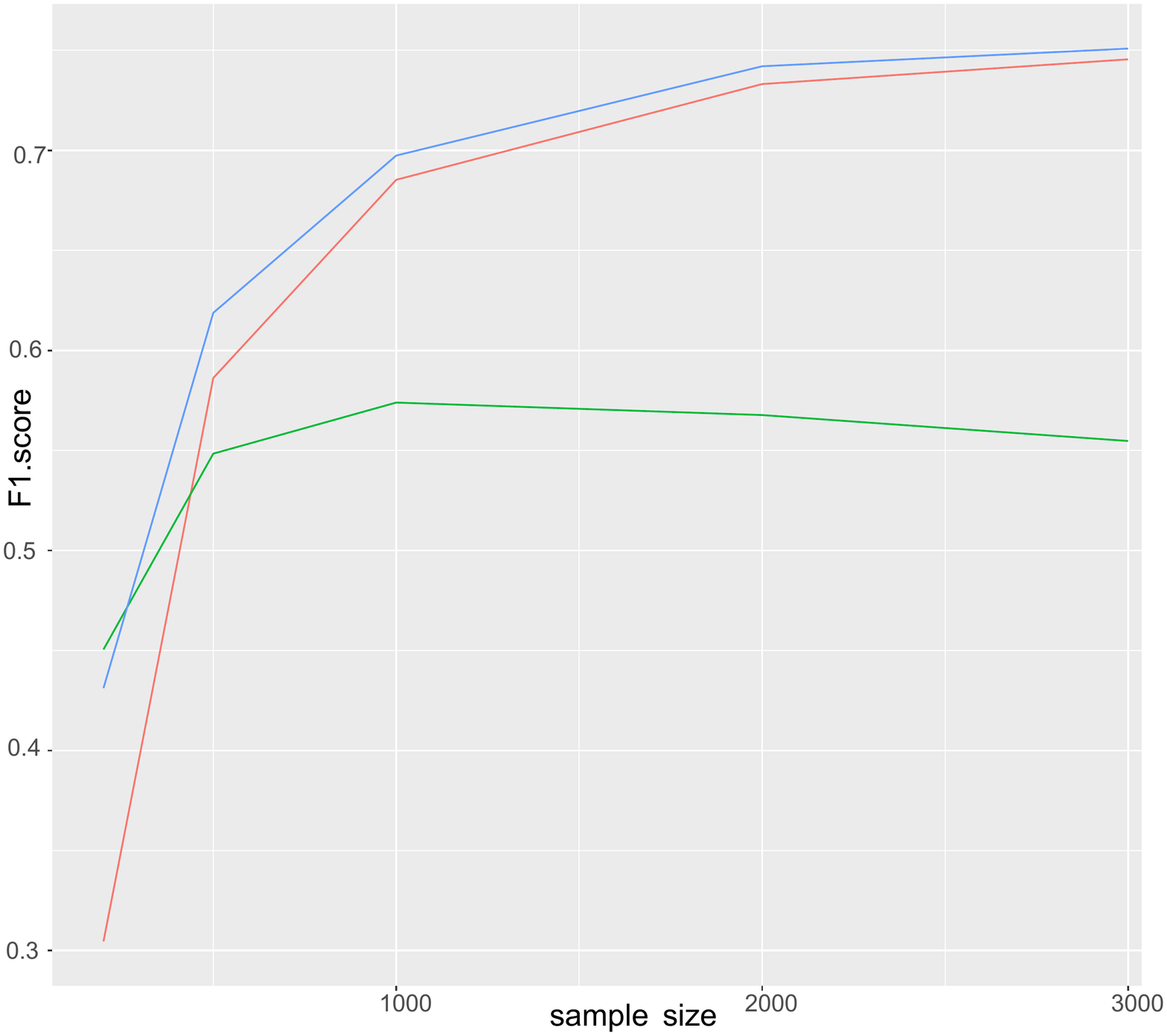} \label{fig:binaryf2}}

\caption{$F_1$ score comparison in discrete graphical models}
\end{figure}

\subsection{Mixed graphical models}
In this section, we simulate two mixed graphical models Markov with respect to the scale-free network and random network as shown in Figure \ref{fig:mix} and Figure \ref{fig:mix2}.

\begin{figure}[H]
\subfloat[Scale free network with 50 variables]{\includegraphics[width=0.5\linewidth]{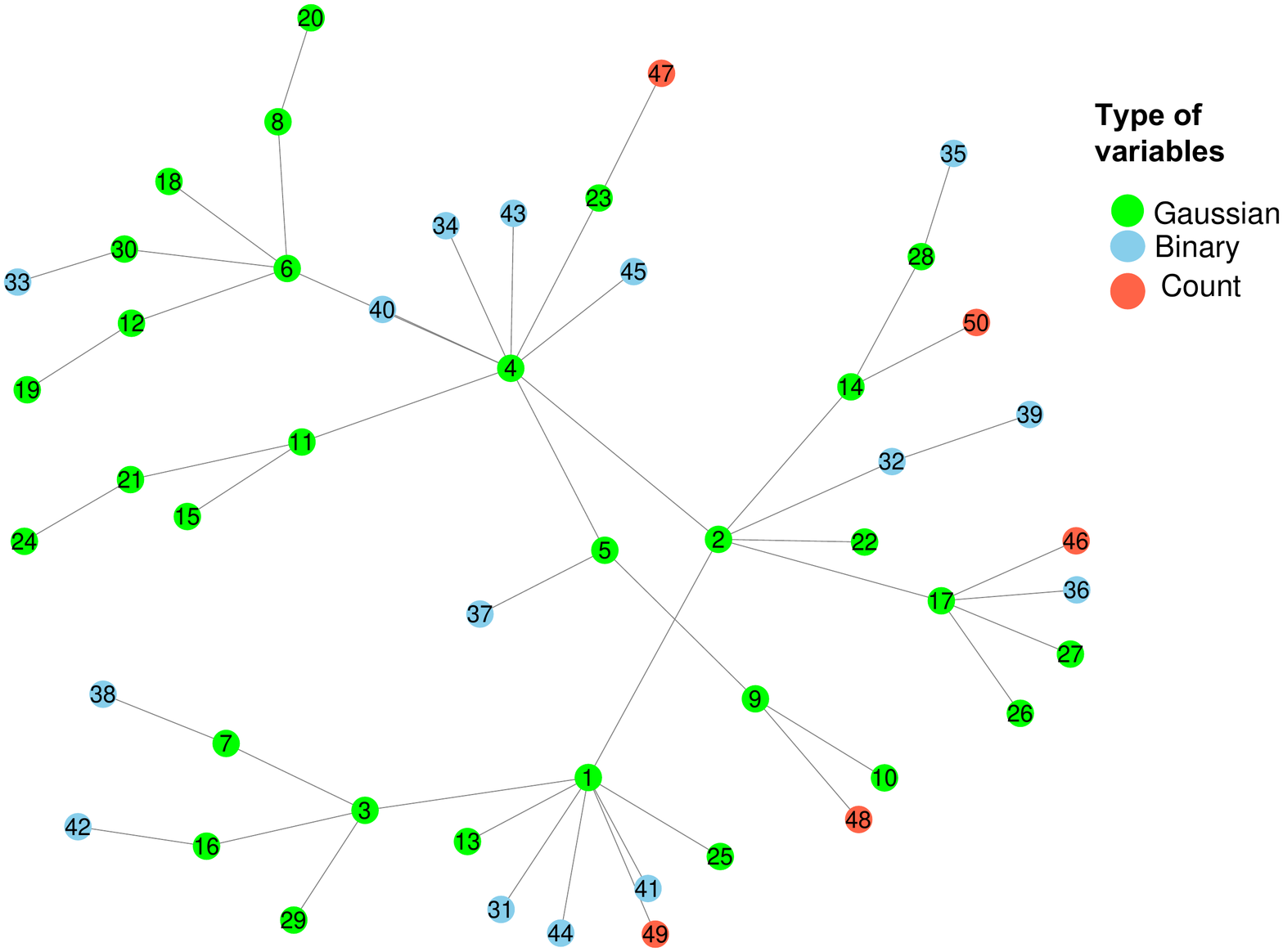}\label{fig:mix}}
\subfloat[Random network with 50 variables]{\includegraphics[width=0.5\linewidth]{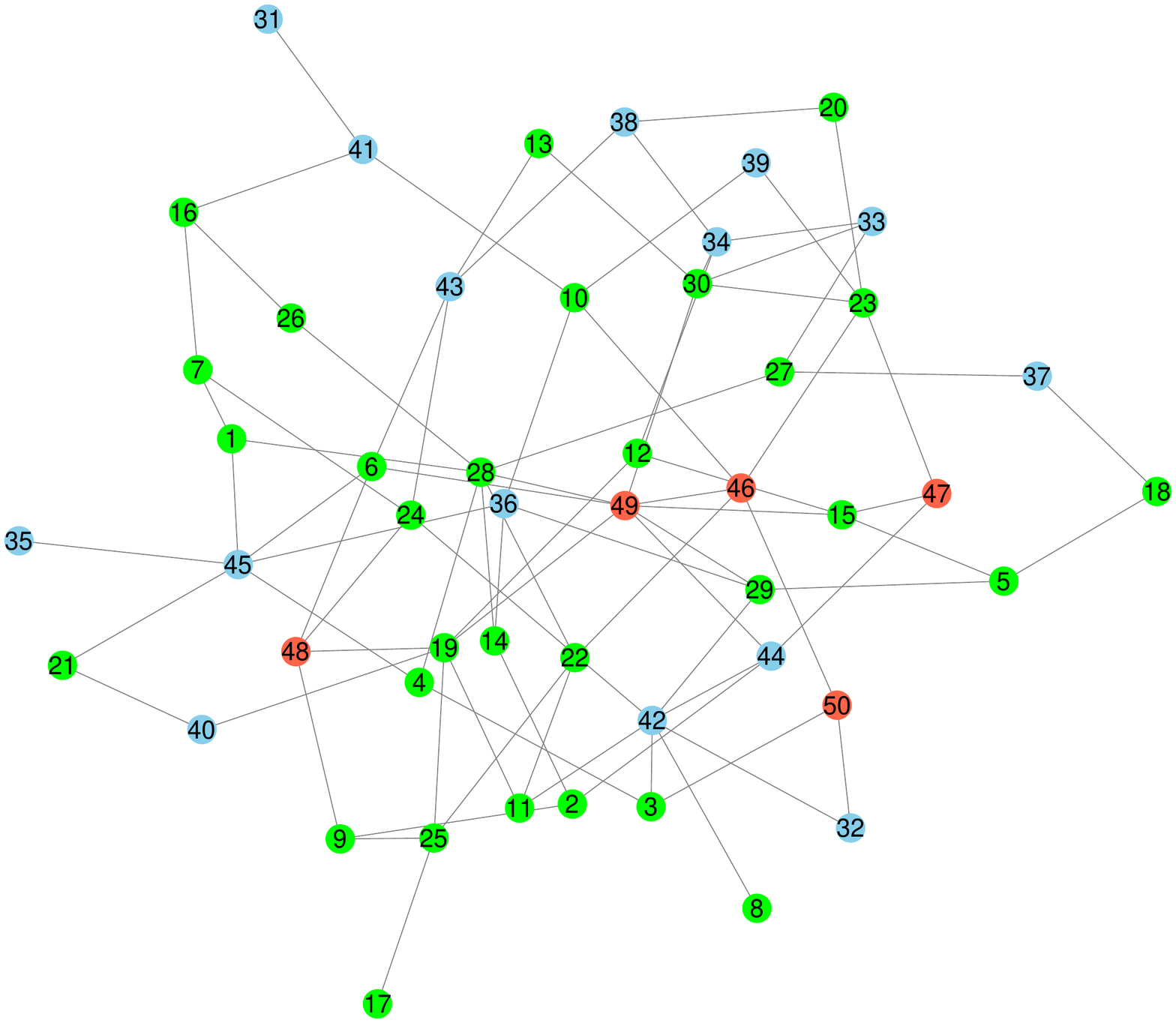}\label{fig:mix2}}
\caption{Graph structures used in the simulation of mixed graphical models }
\end{figure}
The steps for this simulation are as follows:
\begin{enumerate}
\item Generate parameters based on the scale-free or random network graph structure;
\item Use Gibbs sampling to get sample points from the mixed graphical model;
\item Apply the SBDMCMC, BDMCMC, MGM algorithms to learn the graph structure from the given sample and compute the corresponding $F_1$ score.
\end{enumerate}
The $F_1$ score comparison results are shown in Figure \ref{fig:mixf1} and Figure \ref{fig:mixf2}. In mixed graphical models, our SBDMCMC method performs best. 

\begin{figure}[H]
\subfloat[Scale-free network  ]
{\includegraphics[width=0.5 \linewidth]{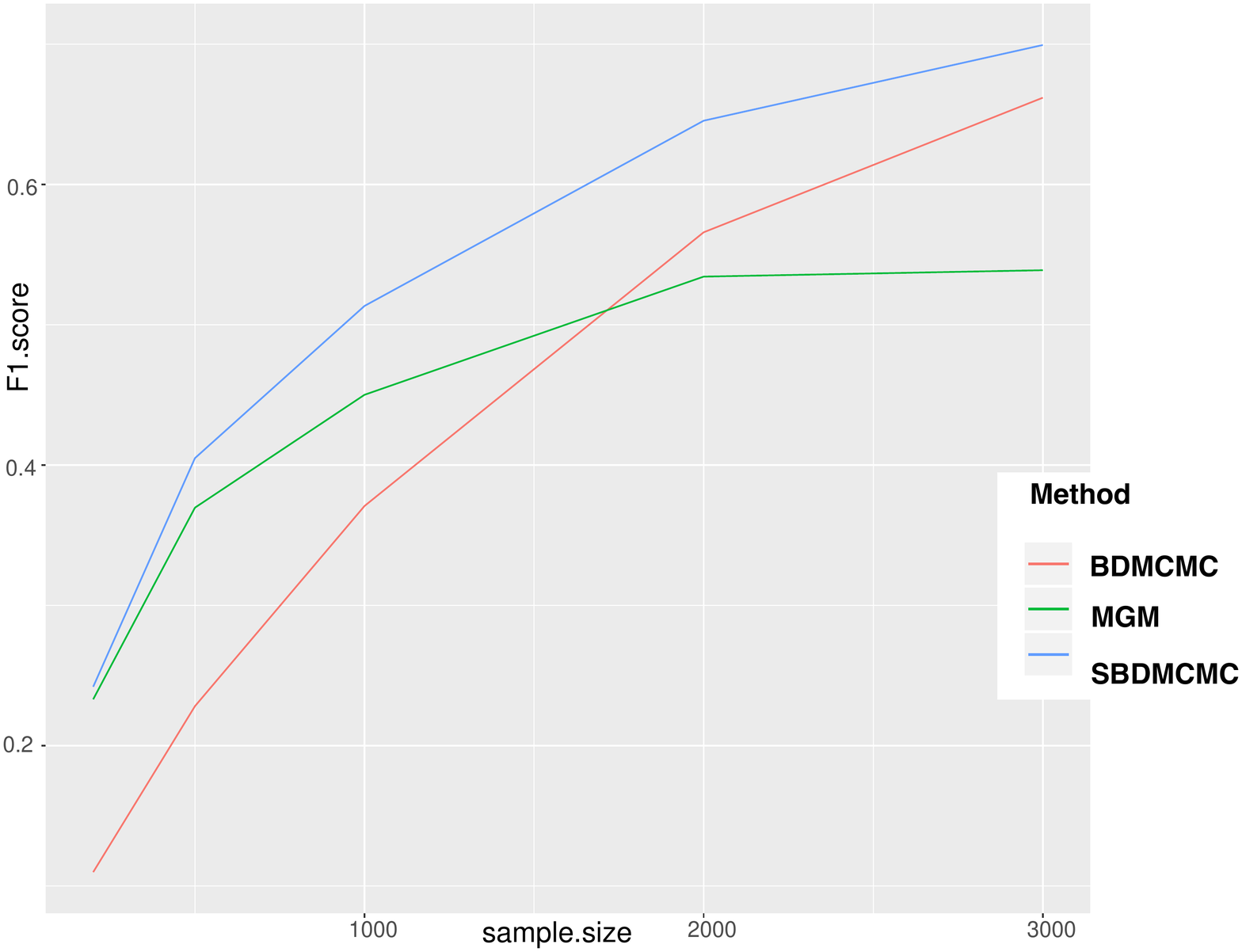} \label{fig:mixf1}}
\subfloat[Random network ]
{\includegraphics[width=0.5 \linewidth]{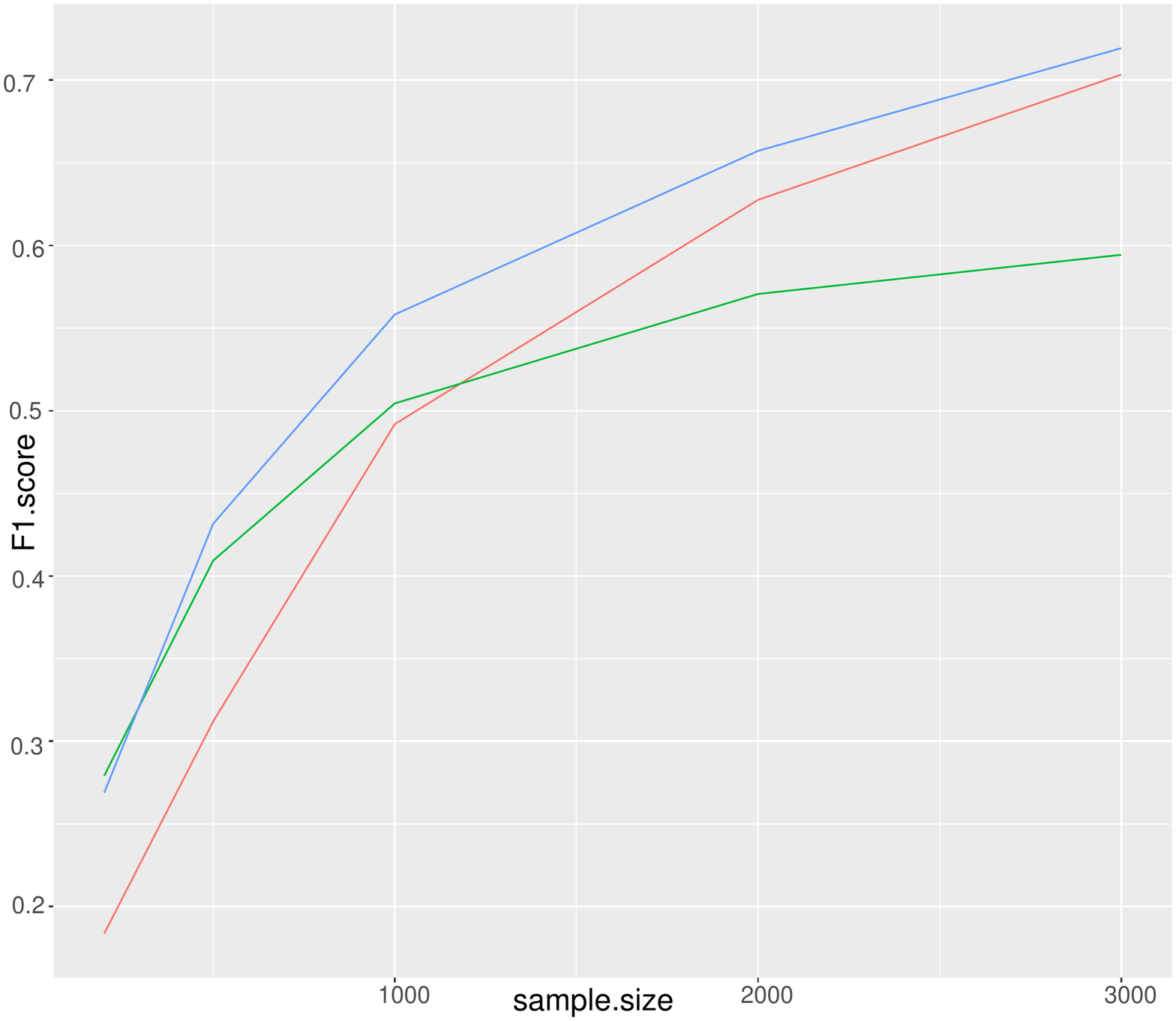} \label{fig:mixf2}}

\caption{$F_1$ score comparison in mixed graphical models}
\end{figure}

\begin{remark}
We comment here on the performance of graphical models to infer the graph structures simulated. The performance of Bayesian model selection methods for graphical models generally depends on two main factors: 1.the approximation accuracy of the marginal likelihood $P(D|G)$; 2. The cardinality of the graph structure space. The original BDMCMC algorithm generates graph structure samples from  the space of entire graphs and this space has cardinality $2^{ \binom{p}{2}}$, where $p$ is the number of variables. Instead, our SBDMCMC algorithm generates neighbourhood structure samples from the neighbourhood structure space of each variable whose cardinality is $2^{p-1}$. Therefore, as $p$ increases, the SBDMCMC algorithm needs to explore a much smaller space than the original BDMCMC algorithm and thus yields better computational efficiency. The approximation accuracy of the marginal likelihood $P(D|G)$ needs to be considered in the different simulation scenarios. In the case of scale free network, the graph structure is close to a decomposable graphical model, so the conditional likelihood estimate used by SBDMCMC is close to the global likelihood estimate. As a consequence, for scale-free networks, the SBDMCMC algorithm outperforms the two competing approaches for model selection. However, in the case of random networks, the graph can be very dense and each local model is not separable from the remaining part of the graph, so the conditional likelihood estimates from the SBDMCMC algorithm do not approximate the global likelihood as well as they do for the scale free networks. Our simulation results thus depend on the tradeoff between computational efficiency and approximation of the marginal likelihood and also on the type of variables that compose the graph.

For instance, with Gaussian graphical models, the BDMCMC algorithm computes the MLE of the parameters of the graphical models globally and is likely more accurate than the conditional likelihood estimate from SBDMCMC. Therefore, the better accuracy of BDMCMC to estimate the marginal likelihood overcomes its cardinality disadvantage compared to SBDMCMC when inferring random networks, as assessed by the $F_1$ score statistic for Gaussian graphical models (see Figure \ref{fig:gaussianf2}).  However, this advantageous tradeoff disappears when inferring Gaussian graphical models with a scale free network graph structure (see Figure \ref{fig:gaussianf1}), where SBDMCMC performs better.
 
 In discrete graphical model scenarios, computing the MLE globally is an NP-hard problem. The modified BDMCMC algorithm in \cite{dobra2018loglinear} uses a pseudo-likelihood to approximate the marginal likelihood $P(D|G)$, so it holds no advantage compared to SBDMCMC to approximate the marginal likelihood and still remains less efficient to investigate a large model space. As a consequence, SBDMCMC performs better than BDMCMC in both scale free networks and random networks for discrete graphical models, see Figure  \ref{fig:binaryf1} and \ref{fig:binaryf2}.
 
 Finally, in the mixed graphical model scenarios, BDMCMC in \citet{dobra2018loglinear} uses Gaussian copula graphical models for datasets with mixed variables. From Figure \ref{fig:mixf1} and \ref{fig:mixf2}, we note that it performs worse than SBDMCMC, and even worse than the  $l_1$-penalized regression method when the sample size is small.
\end{remark}

\subsection{Sensitivity to prior specification}
\label{sec:prior}
In order to assess whether the choice between the three prior distributions defined in Subsection 6.2 affects the results of the SBDMCMC algorithm, we ran several simulations with the mixed graphical models assuming different prior distributions. The setting of this simulation is as follows: first we generate a random graph structure for an undirected graph model with 50 Gaussian variables and 50 Binary variables as shown in Figure \ref{fig:prior_test}; then, we randomly generate  parameter values from a standard normal distribution and generate samples from the mixed graphical model distribution with sample sizes of 100 and 500; Finally, we run our SBDMCMC algorithm and generate ROC plots as displayed in Figure \ref{fig:roc}.

\begin{figure}[H]
\includegraphics[scale=0.4]{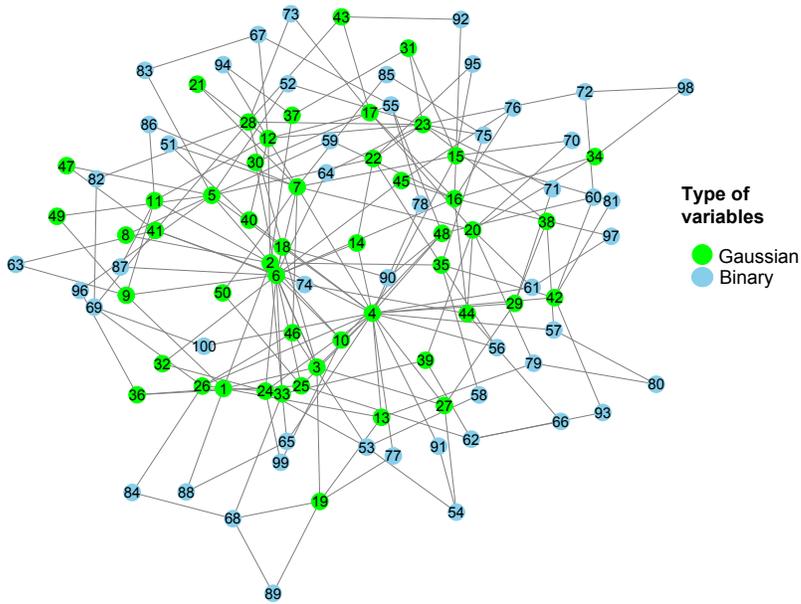}
\caption{Mixed graphical model simulated for testing different priors}
\label{fig:prior_test}
\end{figure} 

\begin{figure}[H]
\subfloat[DM prior with sample size  100]{\includegraphics[width=0.5\linewidth]{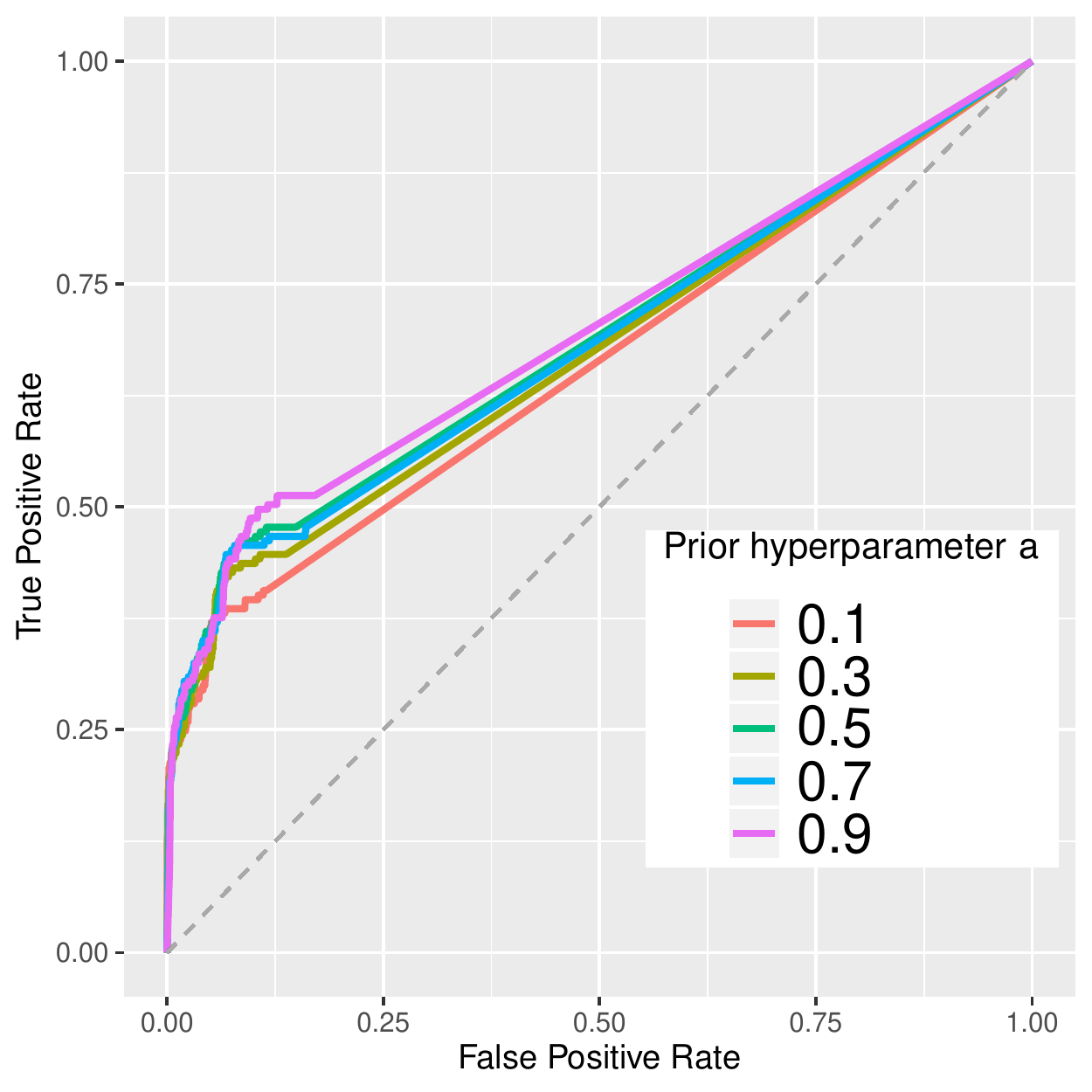}\label{fig:prior01}}
\subfloat[DM prior with with sample size  500]{\includegraphics[width=0.5\linewidth]{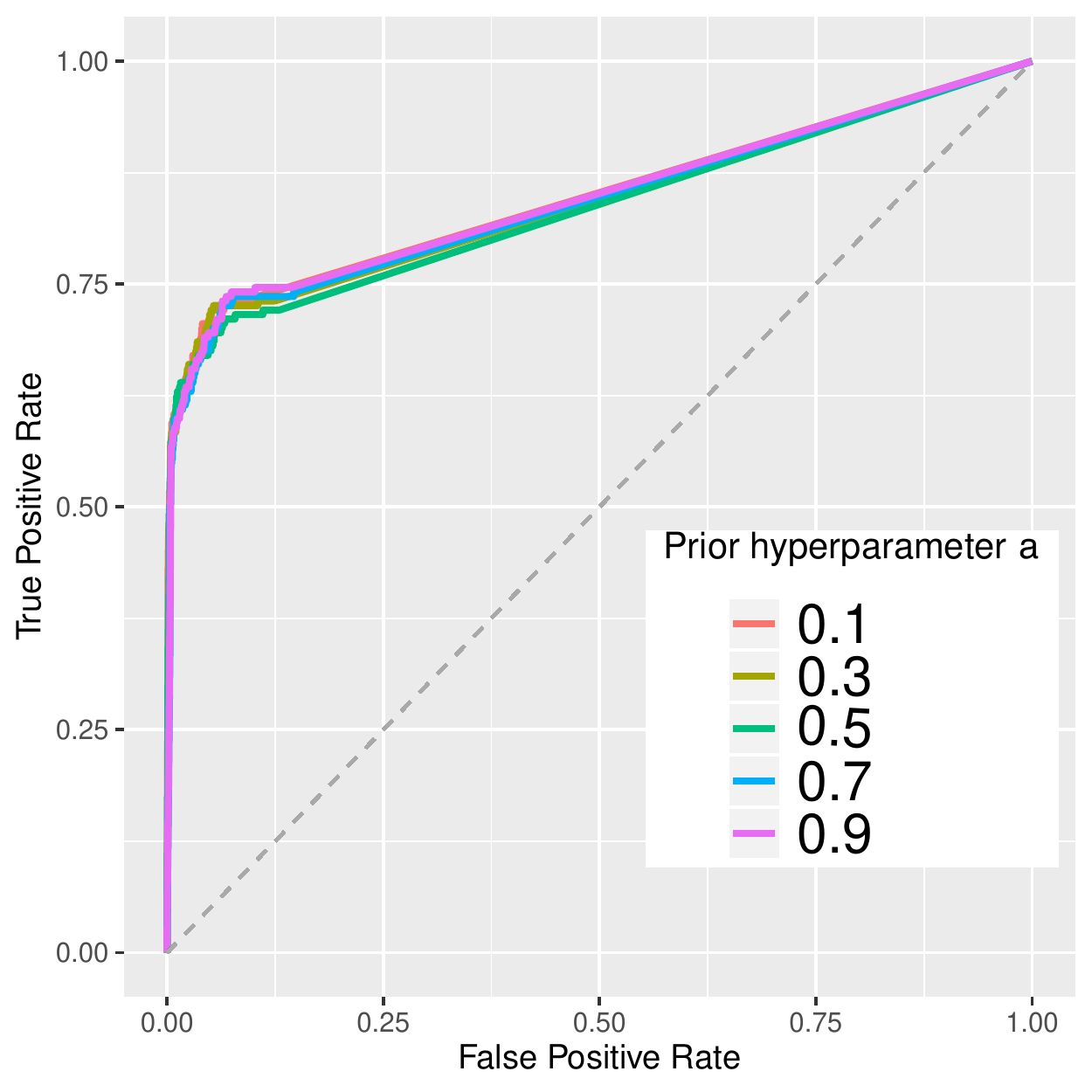}
\label{fig:prior05}}

\subfloat[NY prior with sample size  100]
{\includegraphics[width=0.5\linewidth]{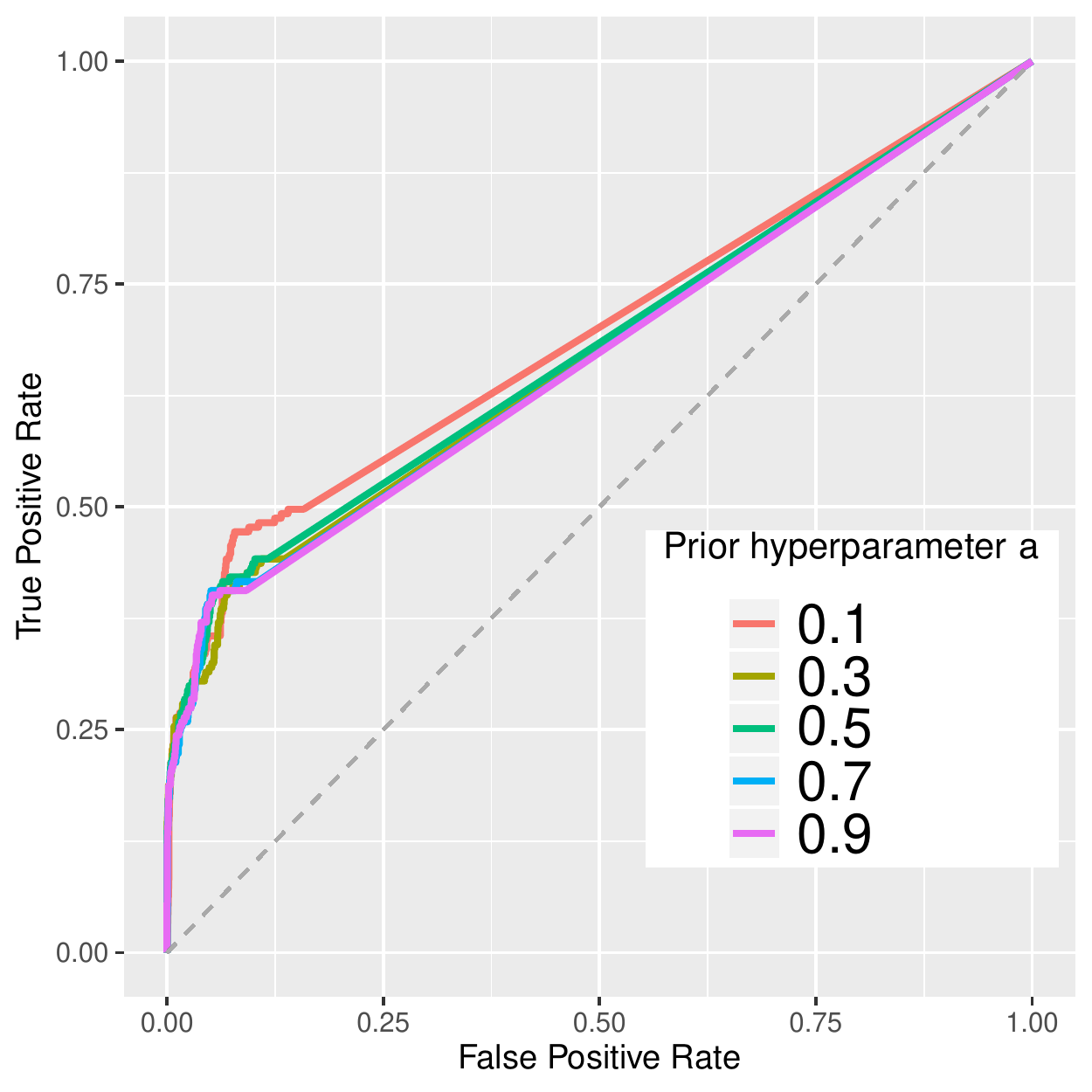}\label{fig:prior11}} 
\subfloat[NY prior with sample size  500]
{\includegraphics[width=0.5\linewidth]{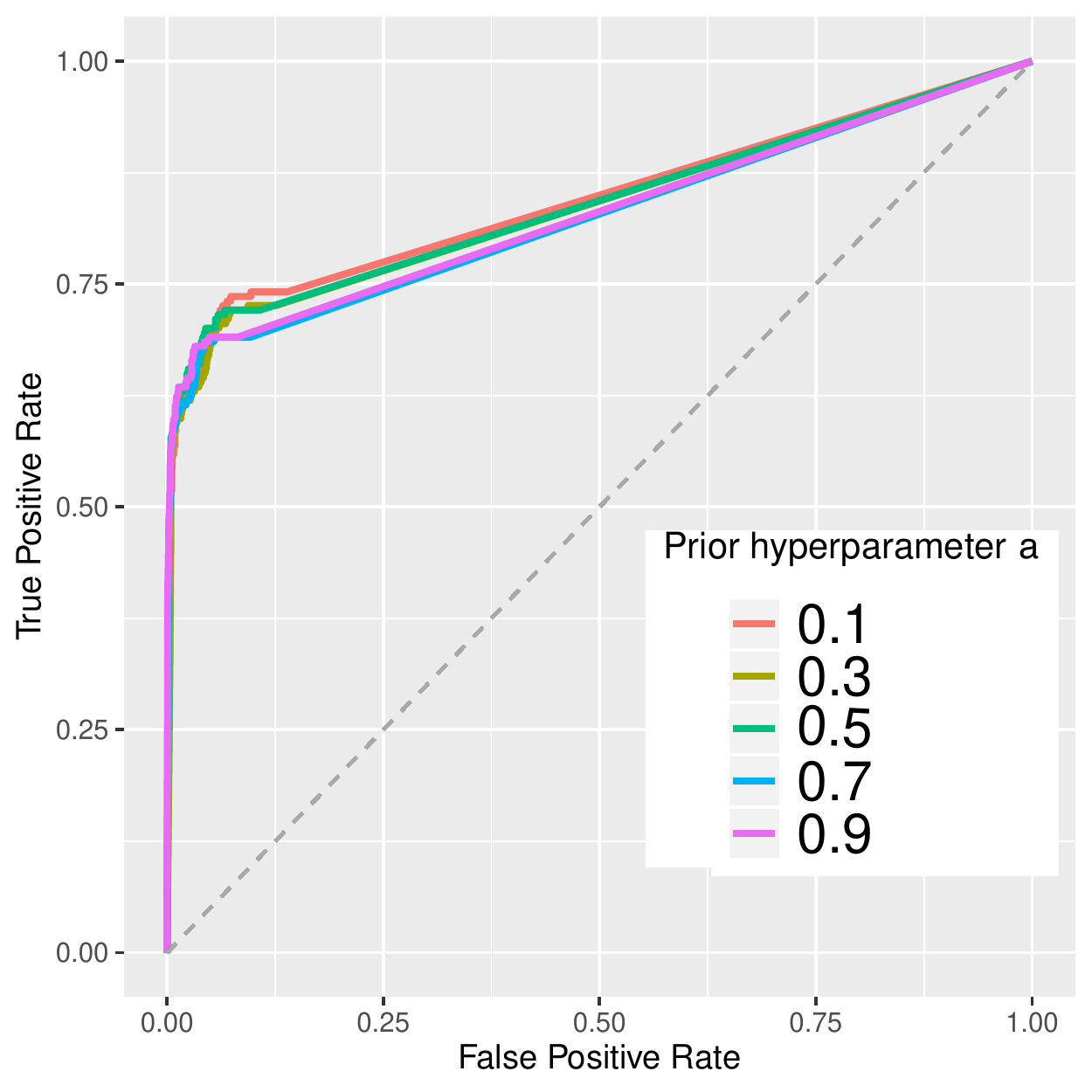}\label{fig:prior15}} 
\end{figure}

\begin{figure}[H]

\subfloat[SB prior with sample size  100]
{\includegraphics[width=0.5\linewidth]{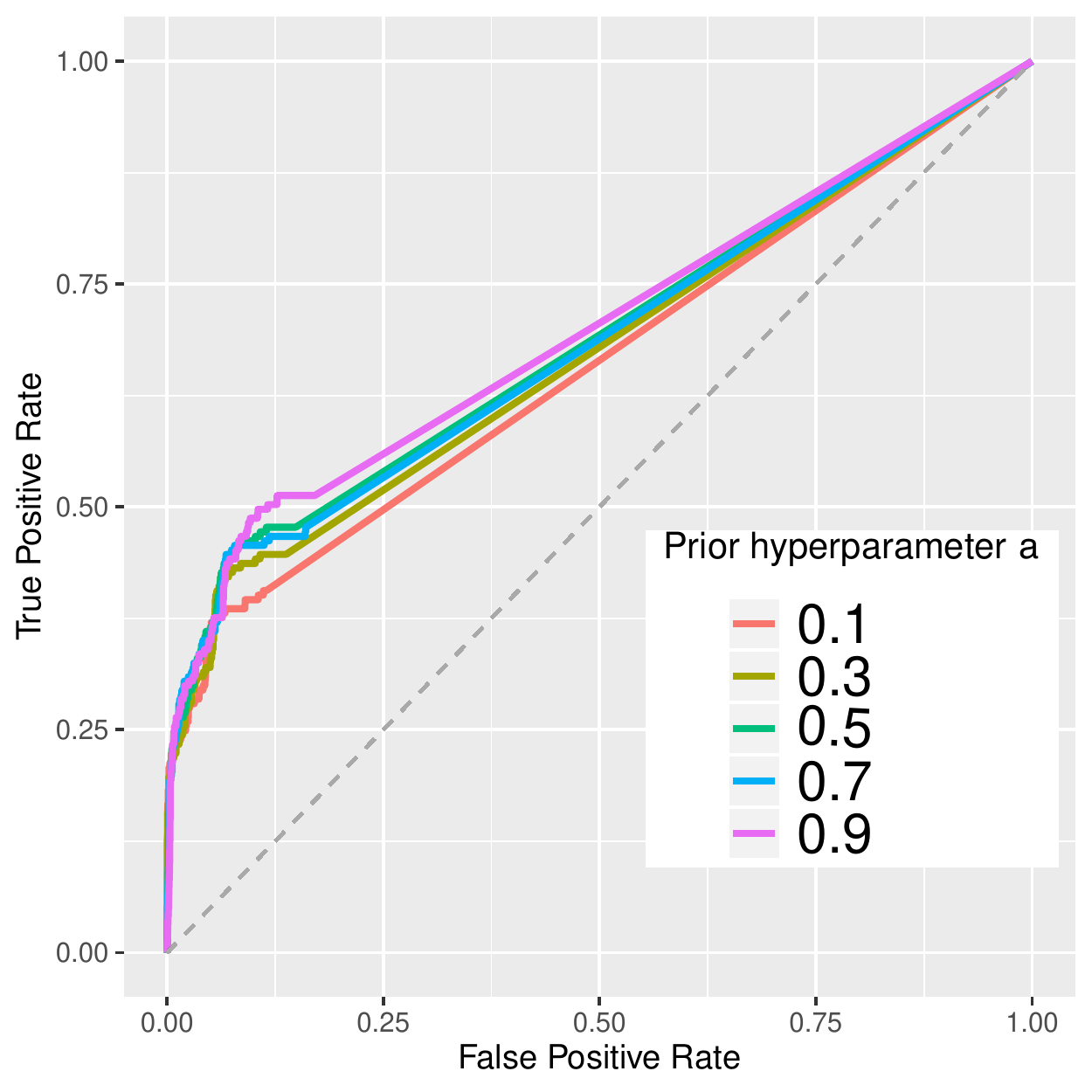}\label{fig:prior21}}
\subfloat[SB prior with sample size  500]
{\includegraphics[width=0.5\linewidth]{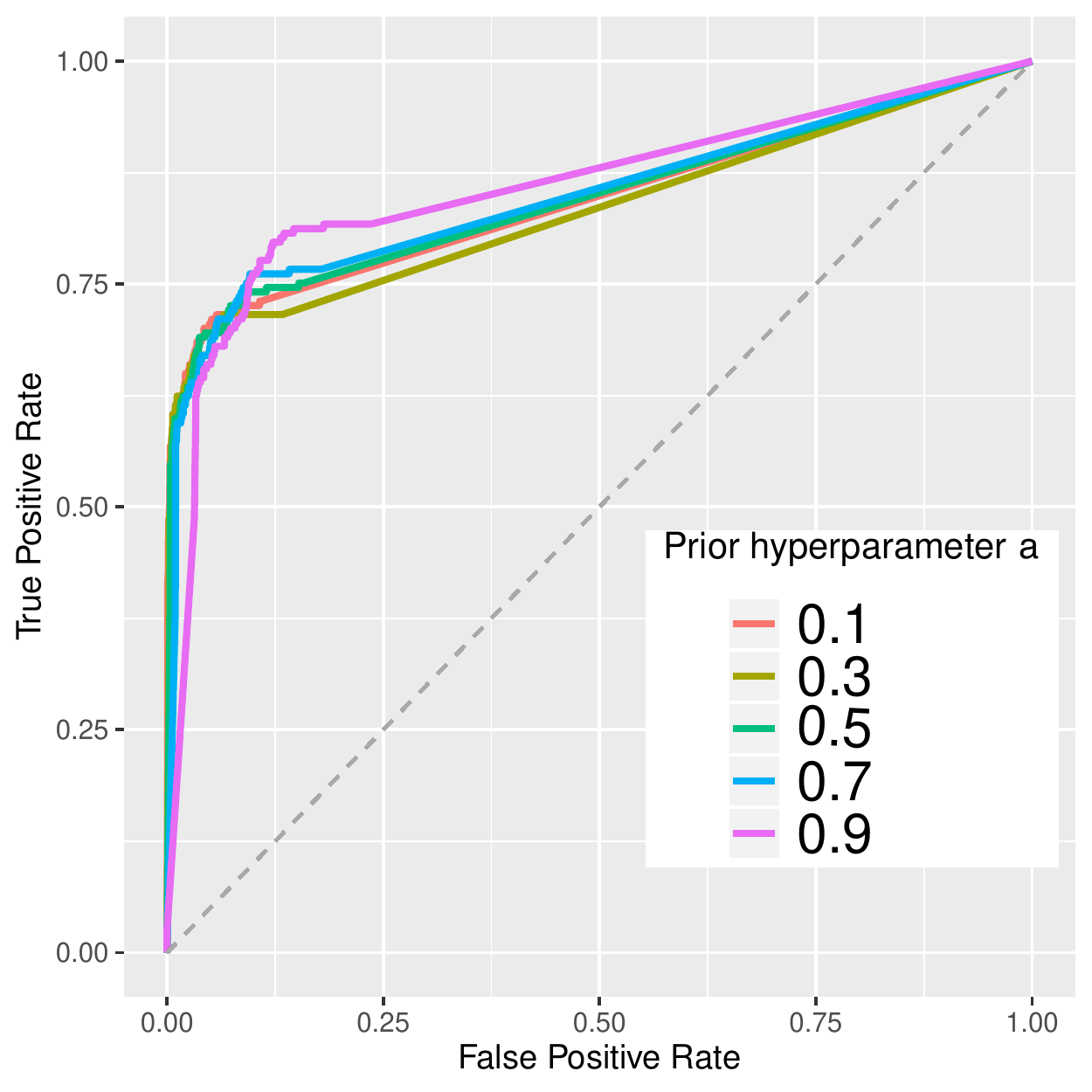}\label{fig:prior25}}
\caption{ROC plots for mixed graphical models with 50 Gaussian variables and 50 binary variables,  with different prior settings and  using the  SBDMCMC algorithm }
\label{fig:roc}
\end{figure}
From these simulation results, we note that the choice of the prior does not affect much the performance of the SBDMCMC algorithm. In all our simulation results and applications below, we use the DM prior with $a=0.5$. 





\section{Real data analysis}
\label{sec:real}
\subsection{Breast cancer}
Breast cancer is one of the most common cancers with greater than
1,300,000 cases and 450,000 deaths each year worldwide. 
Breast cancer is a complex disease that is caused by multiple genetic and
environmental factors. Molecular heterogeneity in breast
tumors has challenged diagnosis, prognosis, and clinical
treatment. Human breast cancer presents both intra- and inter-tumor
molecular heterogeneity. Intertumor variation is manifested by
molecular subtypes that represent significant differences in
prognosis and survival \citep{Parker2009}. Evidence from
gene expression profiling and unsupervised clustering analysis
has indicated five major subtypes for breast carcinomas: 
ER-positive/HER2-negative (luminal A and luminal
B subtypes); ER-negative/HER2-negative (basal subtype); HER2-
positive; and carcinomas that have features similar to normal
breast tissue \citep{Perou1999, Sorlie2001}. 
A set of 50 genes (PAM50) has
been proposed to classify breast tumor samples into the subtypes \citep{Parker2009}. Six (ESR1, PGR, FOXA1, FOXC1,
MYC, and MYBL2) of the 50 genes are transcriptional factor
genes \citep{Parker2009}. Recent studies on genomics and
transcriptomics of large patient populations have identified
additional subtypes \citep{Curtis2012, Guedj2012}. In
the present study, we will focus on the five major subtypes.

\subsection{TCGA data} 
The Cancer Genome Atlas (TCGA): The TCGA consortium, which is a National Institute of Health (NIH)
initiative, makes publicly available molecular and clinical information for more than 30 types of human
cancers including exome (variant analysis), single nucleotide polymorphism (SNP), DNA methylation,
transcriptome (mRNA), microRNA (miRNA) and proteome. Sample types available at TCGA are: primary
solid tumors, recurrent solid tumors, metastatic tumors, blood-derived normal and tumor samples, and normal solid tissues.
TCGA data is accessible via the the NCI Genomic Data Commons (GDC) data portal, GDC Legacy Archive and the
Broad Institutes GDAC Firehose. The GDC Data Portal provides access to the subset of TCGA data that has been
harmonized against GRCh38 (hg38) using GDC Bioinformatics Pipelines which provides methods to the standardization
of biospecimen and clinical data, the re-alignment of DNA and RNA sequence data against a common reference
genome build GRCh38, and the generation of derived data. We used the $R$ package $TCGAbiolinks$ \citep{Colaprico2016} to download the TCGA data, which is freely available through the Bioconductor repository. Somatic mutation information was available on 4864 known cancer genes and 985 breast cancer patients. For each gene, a binary variable was created by indicating whether an insertion or deletion was found (variable coded as 1) or not (variable coded as 0). Low quality and potential germline variants have been removed from this file. Gene expression data were available on 1095 breast cancer patients and 19672 unique genes. We only used expression data from the Illumina HiSeq 2000 RNA Sequencing Version 2 analysis platform corresponding to primary solid tumors. 

\subsection{Aims of this study} 
Certain gene signatures were devised to identify
breast cancer molecular subtypes, which may help in accessing
prognosis. PAM50 is an example of such tests. Of note, multi-gene tests and molecular
subtype classification do not inform us about the mutations and
epigenetic events that have impact in cancer progression. Some
authors argue in favor of assays that are based on the combination
of mutation profiling with the gene expression analysis 
because the presence of specific driver genetic aberrations can
predict the response to specific targeted therapies \citep{Turner2013, Weigelt2011}. Thus,
one important approach is to assess the complete spectrum
of cancer mutations and find the specific actionable molecules
that are crucial in order to perform tailored therapy. Therefore our goals here are first to validate some of the major gene mutations and gene expressions associated with breast cancer subtypes and second, to assess whether an enriched PAM50 expression data can help refining the definition of the BC subtypes.


\subsubsection{PAM50 expression data and the BC subtypes} 
{\bf Analysis:} We used the logarithm of all the gene expressions to get Gaussian distributions for these variables. We constructed mixed graphical models based on Gaussian variables comprised of the 50 gene (PAM50) expression variables and the 5 binary variables which are the  breast cancer subtypes. This analysis included 772 breast cancer women with known cancer subtypes decomposed into 416 luminal A, 141 luminal B, 135 basal, 46 HER2 and 34 normal-like samples. The results are shown in Figure \ref{fig:es1} for the analysis integrating neighbourhoods based on the AND rule and in Figure \ref{fig:es2} on the OR rule. In these two networks, we removed the edges among gene expressions to better visualize the connections between BC subtypes and the PAM50 gene expressions. \\
{\bf Results:} For the graph with the AND rule (Figure \ref{fig:es1}), we first notice that a few genes from the PAM50 list connect two distinct subtypes, including ERRB2, ESR1, FOXC1, NAT1 and SFRP1. This was expected for  ERRB2 and ESR1 since these genes determine directly the BC subtypes (the estrogen receptor is a protein coded by the ESR1 gene and the HER2 protein by the ERBB2 gene). Besides, FOXC1 is known to be a key prognostic indicator of basal-like breast cancer (Jin et al. 2015), NAT1 has been associated with ER+ BCs and the luminal type \citep{Perou1999, Sorlie2001}, and SFRP1 has been shown to be associated with the basal and luminal BC types (Huth et al.  2014). A simple Pearson correlation matrix (Figure \ref{Rplot1}) indicates that these genes have opposite correlation signs with the 2 subtypes they are connected to, except SFPR1 (negatively correlated with lumA and lumB). Some of the edges we discovered confirm previous knowledge about the BC subtypes. For instance, the edge CDC20-basal: CDC20 has recently been identified as the top hit of genes over-expressed in patients with basal-type BC (Sharma et al. 2018); The edge lumA-BCL2: BCL2 expression has been found to be a good prognostic marker for only luminal A breast cancer (Eom et al., 2016); Other edges with lumA (MLPH, MMP11, TYMS. CENPF, NAT1, SFRP1) have also some level of evidence in the literature either specifically with the Lumina A subtype or more generally with ER+ subtypes. The graph with the OR rule (Figure \ref{fig:es2}) confirms that the PAM50 set of genes contribute to the definition of the BC subtypes since they are all connected to more than one subtype.  


\begin{figure}[H]
\subfloat[Mixed network based on the AND rule  ]
{\includegraphics[width=0.5 \linewidth]{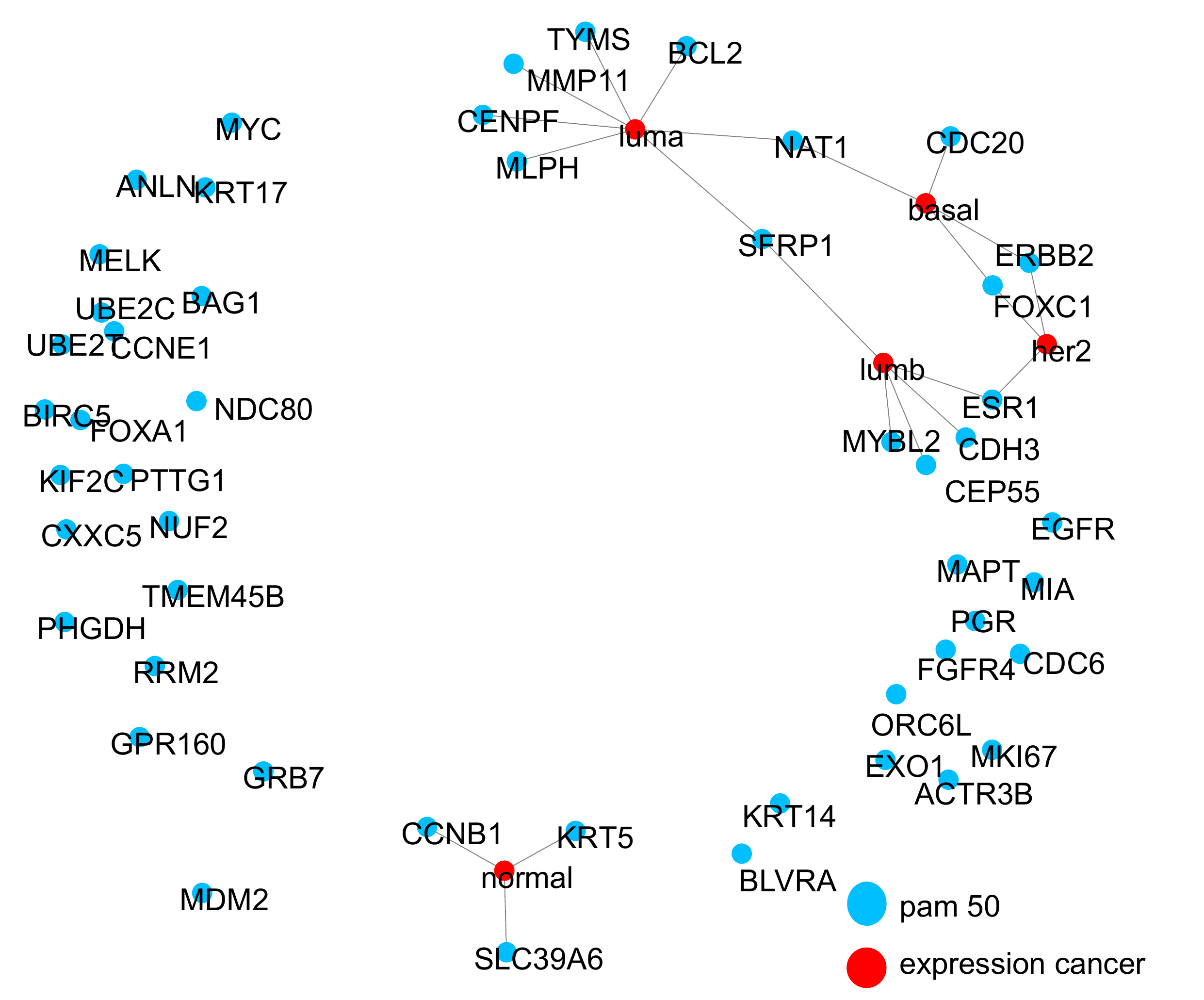} \label{fig:es1}}
\subfloat[Pearson correlation matrix between BC subtypes and selected gene expressions ]
{\includegraphics[width=0.5 \linewidth]{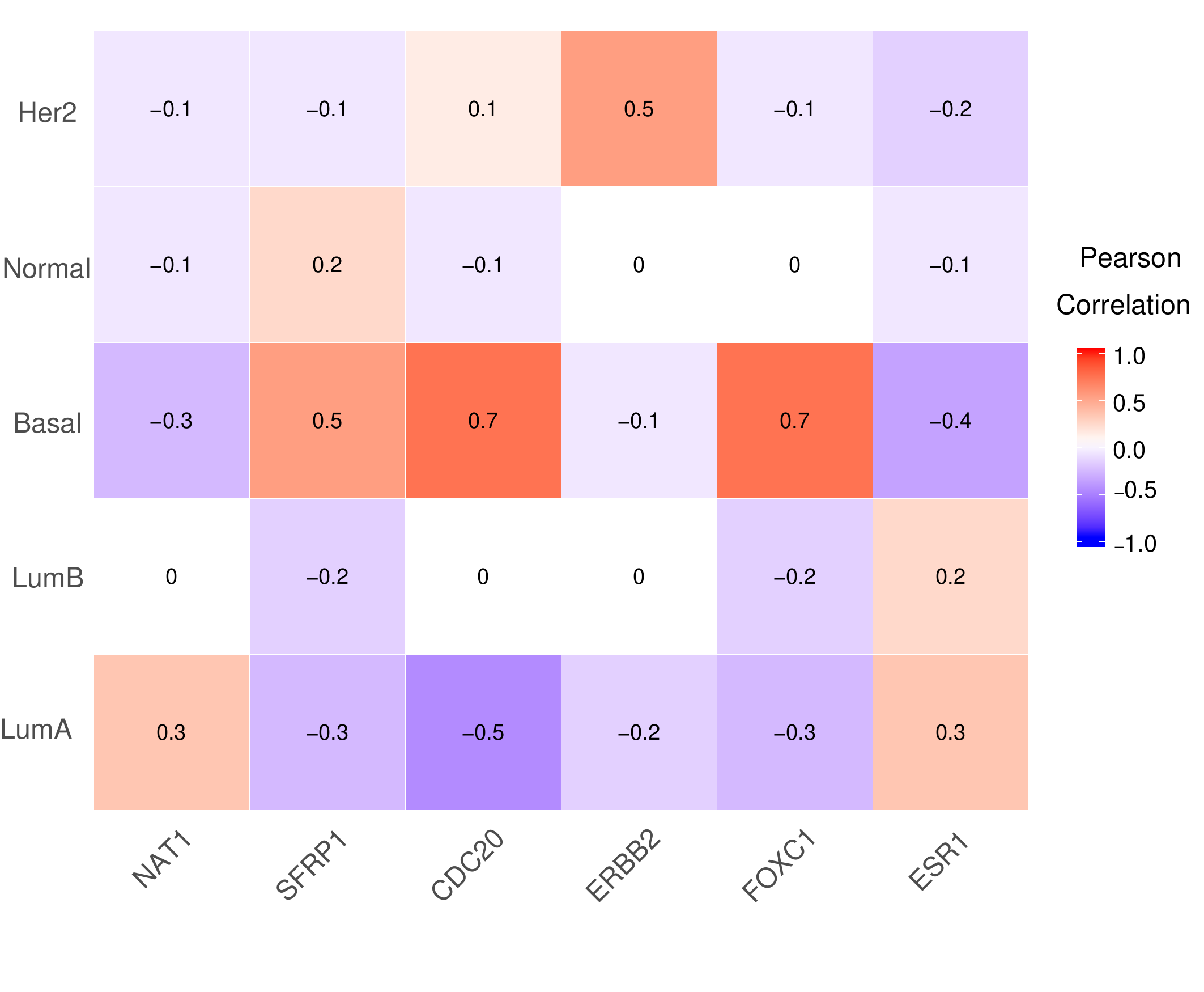} \label{Rplot1}}
\caption{Gene expressions and BC subtypes }
\end{figure}



\begin{figure}[H]
\centering
\includegraphics[scale=0.4]{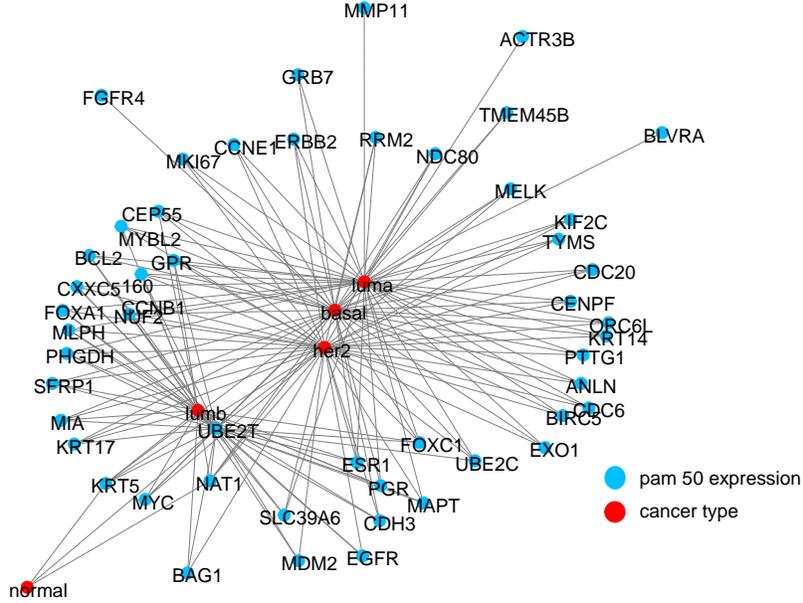}
\caption{Mixed network of gene expressions and BC subtypes based on the OR rule}
\label{fig:es2}
\end{figure}

\subsubsection{Mutation data and the BC subtypes}
{\bf Analysis:} We constructed binary Ising graphical models using the binary variables only, consisting of 46 mutation variables whose variances were greater than 0.03 and the 5 breast cancer subtypes. This analysis included 681 breast cancer women with known subtypes decomposed into 357 luminal A, 135 luminal B, 119 basal, 43 HER2 and 27 normal-like samples. 
{\bf Results:} As depicted in Figures \ref{fig:ms1} and \ref{fig:ms2}, we found similar patterns as for the gene expression and BC subtype network analyses: the normal subtype is disconnected from other BC subtypes. The mixed networks of BC subtype and gene mutation variables based on either the OR rule or the AND rule are very sparse, so here we keep all the connections we found. From the two networks, we can find gene mutations that characterize the four BC subtypes: luminal A, luminal B, basal, and HER2. This includes mutations in the genes CDH1, GATA3, CDH1, PIK3CA, MAP3K1, and TP53. The high prevalence of mutations of TP53 in these four subtypes, of PIK3CA in the luminal A and basal subtypes, of MAP3K1 in luminal A subtype, of CDH1 in the luminal A subtype have been already well described in the recent literature \citet{TCGA2012}. Finally, the normal-like BC subtype does not seem to be determined by any of this panel of 46 mutations.

\begin{figure}[H]
\subfloat[Mixed network based on the AND rule  ]
{\includegraphics[width=0.5 \linewidth]{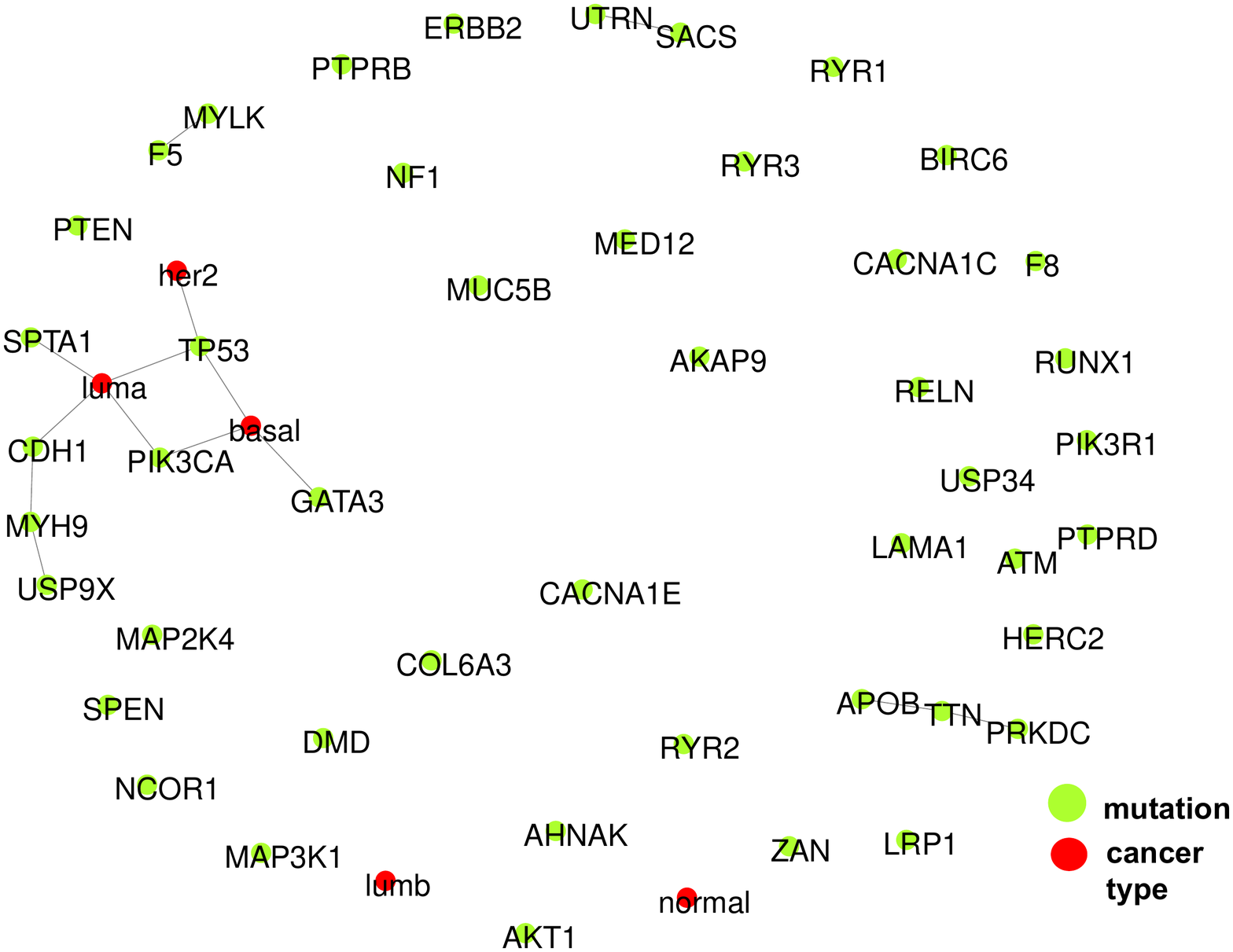} \label{fig:ms1}}
\subfloat[Mixed network based on the OR rule  ]
{\includegraphics[width=0.5 \linewidth]{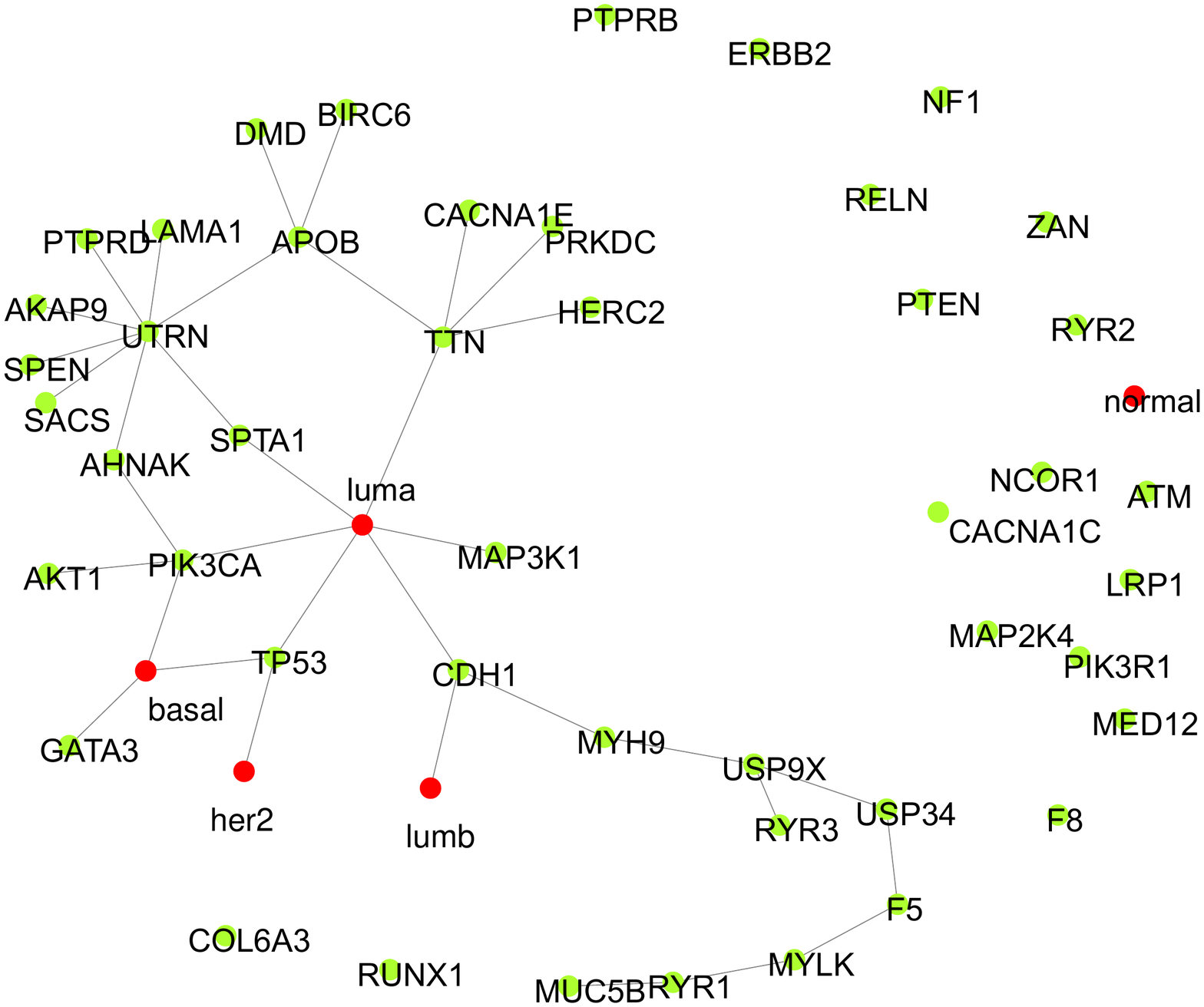} \label{fig:ms2}}
\caption{Mixed network of gene mutations and BC subtypes}
\end{figure}



\subsubsection{Mutation data, enriched PAM50 expression data and the BC subtypes}
{\bf Analysis:}  We constructed mixed graphical models based on the mutation binary variables described above,  the 50 Gaussian gene expression variables defining the PAM50 panel and the five BC subtypes. We also enrich the PAM50 expression panel with 15 genes selected using the following procedure: 1) We performed univariate analyses between each BC subtype and the 19,672 gene expression using logistic regression and selected those that passed a Bonferroni corrected p-value of 5\%; 2) We then built a graphical model of the BC subtypes and the 10,574 significant genes and selected only 20 genes that were connected to any of the BC subtypes. Since 5 of them overlapped with the PAM50 panel, we added only 15 genes to this panel. The final analysis included 681 women with known BC subtypes. {\bf Results:} Figure \ref{fig:em} illustrates the main advantage of this multi-omic analysis, which is to identify mutations that are driving the PAM50 expression and sub-typing of breast cancers. The driver mutations are related to the genes SPTA1, AKT1, PIK3CA, GATA3, MAP3K1, CDH1, TP53, ATM, NCOR1, CACNA1C, ERBB2 and TTN. Figure \ref{fig:em3}, which is a zoom-in version of Figure \ref{fig:em}, shows that the inclusion of additional gene expressions also adds information about new driver mutations such as NCOR1 mutation, which regulates the expression of the gene ASCC3 but it does not seem that these additional gene expressions are correlated with the BC subtypes.

\begin{figure}[H]
\centering
\includegraphics[width=\linewidth]{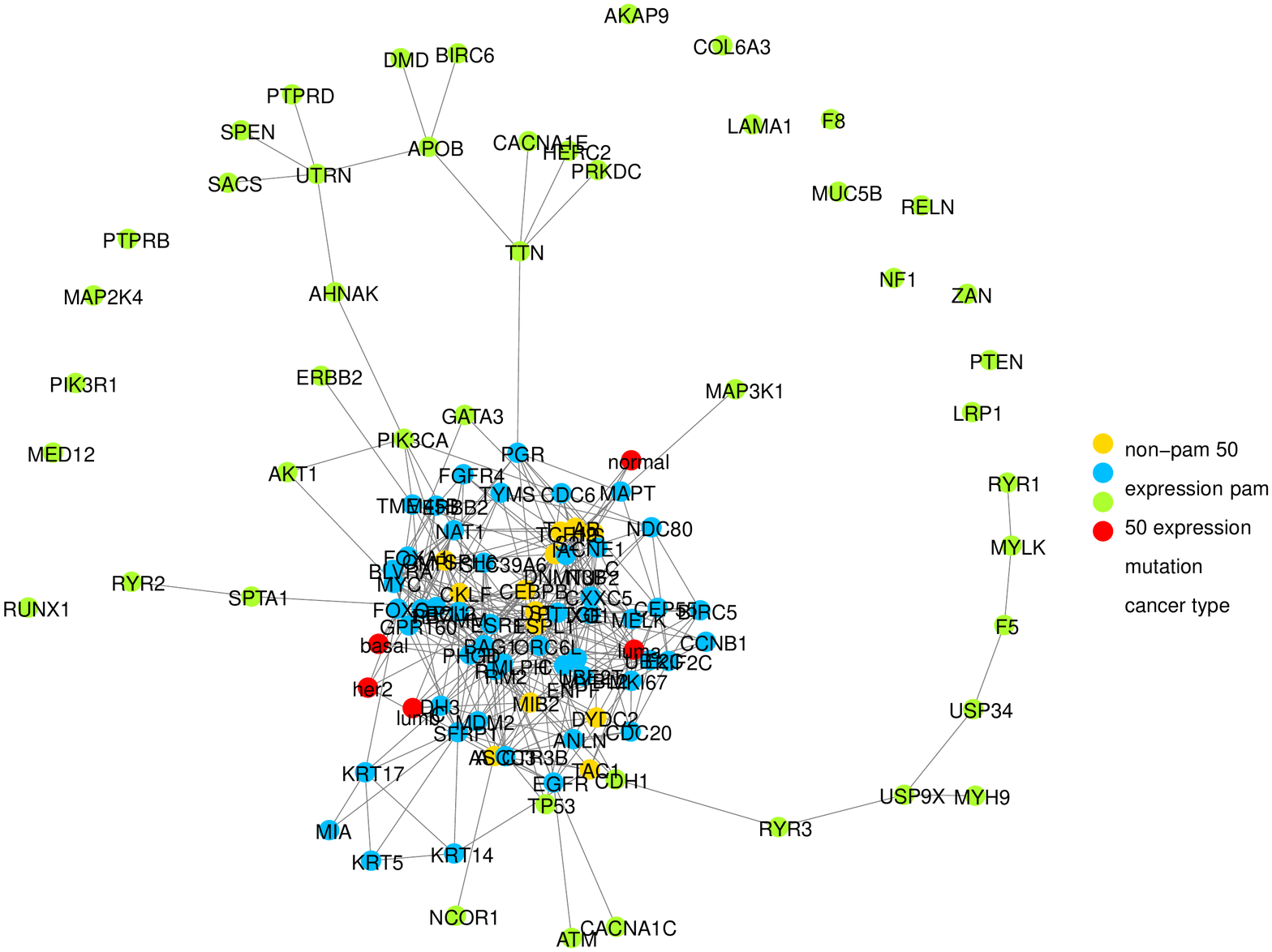}
\caption{The network of the mutation data, enriched PAM50 expression data and the BC subtypes}
\label{fig:em}
\end{figure}


\begin{figure}[H]
\centering
\includegraphics[width=\linewidth]{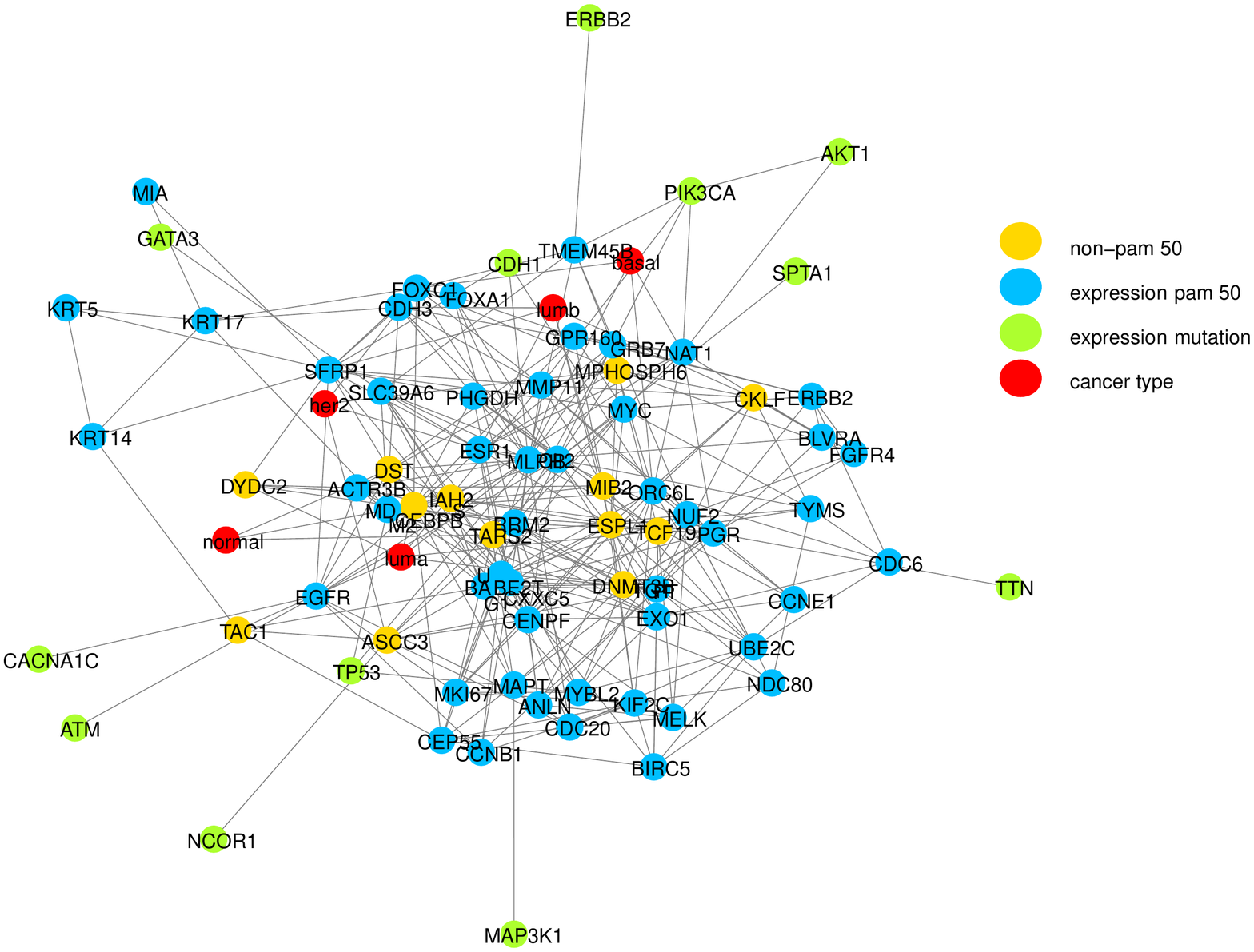}
\caption{Zoom-in version of the network of the mutation data, enriched PAM50 expression data and the BC subtypes}
\label{fig:em3}
\end{figure}


\pagebreak

\section{Discussion and future work}
In this paper, we proposed  a novel mixed graphical model to analyze multi-omic data of different types (continuous, discrete and count) and performed model selection by extending the Birth-Death MCMC (BDMCMC) algorithm initially proposed by \citet{stephens2000bayesian} and later developped  by \citet{mohammadi2015bayesian} and \citet{dobra2018loglinear}. First, we used the BIC and extended BIC to approximate the marginal likelihood $p(D|M)$ needed to perform model selection with the SBDMCMC algorithm; Second, for  graphical model selection, we applied local neighbourhood selection  instead of global structure learning. The marginal likelihood approximation by BIC allows to handle various types of variables in model selection problems while the local neighbourhood search improves the computational efficiency of the MCMC.  In our simulations, one needs less than 100 steps for the SBDMCMC algorithm to converge compared to thousands of steps under the standard BDMCMC algorithm \citet{mohammadi2015bayesian} and \citet{dobra2018loglinear}. \\
\indent Our simulation studies assessed the good performance of the SBDMCMC approach both in terms of computational efficiency and accuracy of model selection when compared to competing approaches such as the original BDMCMC algorithm or a neighbourhood selection approach based on $l_1$-penalized regression method (MGM).  The SBDMCMC approach performs better for inferring  both scale free networks and random networks when the networks contain discrete and mixed types of variables. When the network is only composed of Gaussian variables, the BDMCMC algorithm performed slightly better than the SBDMCMC but only in the situation of random graphs. \\
\indent  Our real data application illustrated the interest of the SBDMCMC approach to analyze multi-omic data from the TCGA consortium. Our analyses were able to validate some of the major gene mutations and gene expressions associated with breast cancer subtypes. In addition, the analysis of mutation data, enriched PAM50 expression data and the BC subtypes showed that we could discover new driver mutations such as NCOR1, which could help refine the definition of the BC subtypes. \\
\indent In our future work, we would like to investigate further whether the model prior $\pi(M)$ plays an important role in selecting a good sparse model in high-dimensional data. We would like therefore to study the choice of prior and how the prior affects the model selection results. A prior driven by the given data could be of interest for instance. Our applications show the interest of the SBDMCMC method when inferring graph structures around known BC subtypes. It would be of interest to extend this approach to also perform new subtype discoveries or refine the existing BC subtypes. \\
\indent All the R codes used to run the simulations and real data analysis experiments can be found in GitHub:

 \url{https://github.com/wangnanwei/Birth-death-MCMC-Model-Selection}


\bibliographystyle{apalike}
\bibliography{ref}
\end{document}